\documentclass{article}
\usepackage{tikz}

\newcommand{\II}{{\mathcal I}}

\usepackage{amsmath,amssymb,amsthm,bm}
\usepackage{thmtools,thm-restate}

\newcommand{\real}{\mathbb{R}}

\newcommand{\V}[1]{{\mathbf{#1}}}

\newcommand{\abs}[1]{ {\left| #1 \right|}}

\newcommand{\sst}{SST}
\newcommand{\imagenet}{ImageNet}

\newcommand{\bert}{BERT}
\newcommand{\resnet}{ResNet152}

\newcommand{\framework}{{\tt Archipelago}}
\newcommand{\archattribute}{{\tt ArchAttribute}}
\newcommand{\archdetect}{{\tt ArchDetect}}

\usepackage{bbm}
\usepackage{scalerel}
\usepackage{breqn}
\usepackage{booktabs}
\usepackage{makecell} 
\usepackage{lipsum}  
\usepackage{wrapfig}
\usepackage{mathtools}
\usepackage{nth}
\usepackage{soul}
\usepackage{cite}
\usepackage{multirow}
\usepackage{float}%
\usepackage[colorlinks=true, linkcolor=blue, citecolor=blue]{hyperref}
\usepackage{colortbl}
\definecolor{darkblue}{rgb}{0.0,0.0,0.5}
\definecolor{lightblue}{rgb}{0.0,0.0,0.8}
\definecolor{LightSteelBlue3}{rgb}{0.635,0.71, 0.804}
\newcommand{\foo}{\color{LightSteelBlue3}\makebox[0pt]{\textbullet}\hskip-0.5pt\vrule width 1pt\hspace{\labelsep}}
\usepackage[preprint,nonatbib]{neurips_2020}
\usepackage[utf8]{inputenc} %
\usepackage[T1]{fontenc}    %
\usepackage{url}            %
\usepackage{booktabs}       %
\usepackage{amsfonts}       %
\usepackage{nicefrac}       %
\usepackage{microtype}      %
\usepackage{subcaption}
\usepackage{graphicx}

\makeatletter
\renewcommand\subsubsection{\@startsection{subsubsection}{3}{\z@}%
                                     {-3.25ex\@plus -1ex \@minus -.2ex}%
                                     {-1.5ex \@plus -.2ex}%
                                     {\normalfont\normalsize\bfseries}}
\makeatother

\title{How does this interaction affect me? \\ Interpretable attribution for feature interactions}

\author{%
  Michael Tsang, Sirisha Rambhatla, Yan Liu\\
  Department of Computer Science\\
  University of Southern California\\
  \texttt{\{tsangm,sirishar,yanliu.cs\}@usc.edu} \\
}

\begin{document}

\maketitle
\vspace{-0.01in}

\begin{abstract}
\vspace{-0.01in}

Machine learning transparency calls for interpretable
explanations of how inputs relate to predictions. Feature attribution is a way to analyze the impact of features on predictions. Feature \emph{interactions} are the contextual dependence between features that jointly impact predictions. There are a  number of methods that extract feature interactions in prediction models; however, the methods that assign attributions to interactions  are either uninterpretable, model-specific, or non-axiomatic.
We propose an interaction attribution and detection framework called {\framework} which addresses 
these problems and is also scalable in real-world settings.
Our experiments on standard annotation labels indicate  our approach provides significantly more interpretable explanations than comparable methods, which is important for 
analyzing the impact of interactions on predictions.
We also provide accompanying visualizations of our approach that give new insights into deep neural networks.

\end{abstract}
\vspace{-0.05in}
\section{Introduction}
\vspace{-0.02in}

The success of state-of-the-art prediction models such as neural networks is driven by their capability to learn complex feature interactions. 
When such models are used to make predictions for users, we may want to know how they personalize to us. Such model behaviors can be explained via \emph{interaction detection}  and \emph{attribution}, i.e. if features influence each other and how these interactions contribute to predictions, respectively. Interaction explanations are useful for applications such as sentiment analysis~\cite{murdoch2018beyond}, image classification~\cite{tsang2020feature}, and recommendation tasks~\cite{tsang2020feature,guo2017deepfm}.

Relevant methods for attributing predictions to feature interactions are black-box explanation methods based on axioms (or principles), but these methods lack interpretability. One of the core issues is that an interaction's importance is not the same as its attribution. Techniques like Shapley Taylor Interaction Index (STI)~\cite{dhamdhere2019shapley} and Integrated Hessians (IH)~\cite{janizek2020explaining} combine these concepts in order to be axiomatic. Specifically, they base an interaction's attribution on non-additivity, i.e. the degree that features non-additively affect an outcome. While non-additivity can be used for interaction detection, it is not interpretable as an attribution measure as we see in Fig.~\ref{fig:motiv}. In addition, neither STI nor IH is tractable for higher-order feature interactions~\cite{sorokina2008detecting,dhamdhere2019shapley}. Hence, there is a need for interpretable, axiomatic, and scalable methods for interaction attribution and corresponding interaction detection.

To this end, we propose a novel framework called {\framework}, which consists of an interaction attribution method, {\archattribute}, and a corresponding interaction detector, {\archdetect}, to
address  the challenges of being interpretable, axiomatic, and scalable.
{\framework}
is named after its ability to provide explanations by  isolating feature interactions, or feature ``islands''.
The inputs to {\framework} are a black-box model $f$ and data instance $\V{x}^{\star}$, and its outputs are a set of interactions and individual features $\{\II\}$ as well as an attribution score $\phi(\II)$ for each of the feature sets $\II$. 

{\archattribute} satisfies attribution axioms by making relatively mild assumptions:
a) disjointness of interaction sets, which is easily obtainable, and
b) the availability of a generalized additive function which is a good approximator to any function, as is leveraged in earlier works~\cite{tsang2017detecting,tsang2018neural,tsang2018can}. On the other hand, 
{\archdetect} circumvents intractability issues of higher-order interaction detection by removing certain uninterpretable higher-order interactions and leveraging a property of feature interactions that allows pairwise interactions to merge for disjoint arbitrary-order interaction detection. In practice, where any assumptions may not hold in real-world settings, {\framework} still performs well.
In particular, {\framework} effectively detects relevant interactions and  is more interpretable than state-of-the-art methods~\cite{dhamdhere2019shapley,jin2019towards,sundararajan2017axiomatic,janizek2020explaining,tsang2018can,grabisch1999axiomatic} when evaluated on  annotation labels in sentiment analysis and image classification. We visualize  {\framework}  explanations on sentiment analysis, COVID-19 prediction on chest X-rays, and ad-recommendation. 

Our main contributions are summarized below.
\vspace{-0.05in}
\begin{itemize}
    \item \textbf{Interaction Attribution:} We propose {\archattribute}, a feature attribution measure that leverages feature interactions.
    It has  advantages of being model-agnostic, interpretable, and runtime-efficient as compared to other state-of-the-art interaction attribution methods.
    \item \textbf{Principled Attribution:} {\archattribute} obeys standard  attribution axioms~\cite{sundararajan2017axiomatic} that are generalized to work for feature sets, and we also propose a new axiom for interaction attribution to respect the additive structure of a function.
    \item \textbf{Interaction Detection:} We propose a complementary feature interaction detector, {\archdetect}, that 
    is also model-agnostic and $\mathcal{O}(p^2)$-efficient  for pairwise and disjoint arbitrary-order interaction detection ($p$ is number of features).
    
\end{itemize}
\vspace{-0.05in}

Our empirical studies on
        {\archdetect} and {\archattribute}
        demonstrate
        their superior properties as compared to state-of-the-art methods.

\begin{figure}[t]
\centering
\scalebox{0.98}{
\begin{tikzpicture}
\node[anchor=south west,inner sep=0] (image) at (0.0,0) {\centering
\includegraphics[scale=0.36,trim={0 0 16.6cm 0cm},clip]{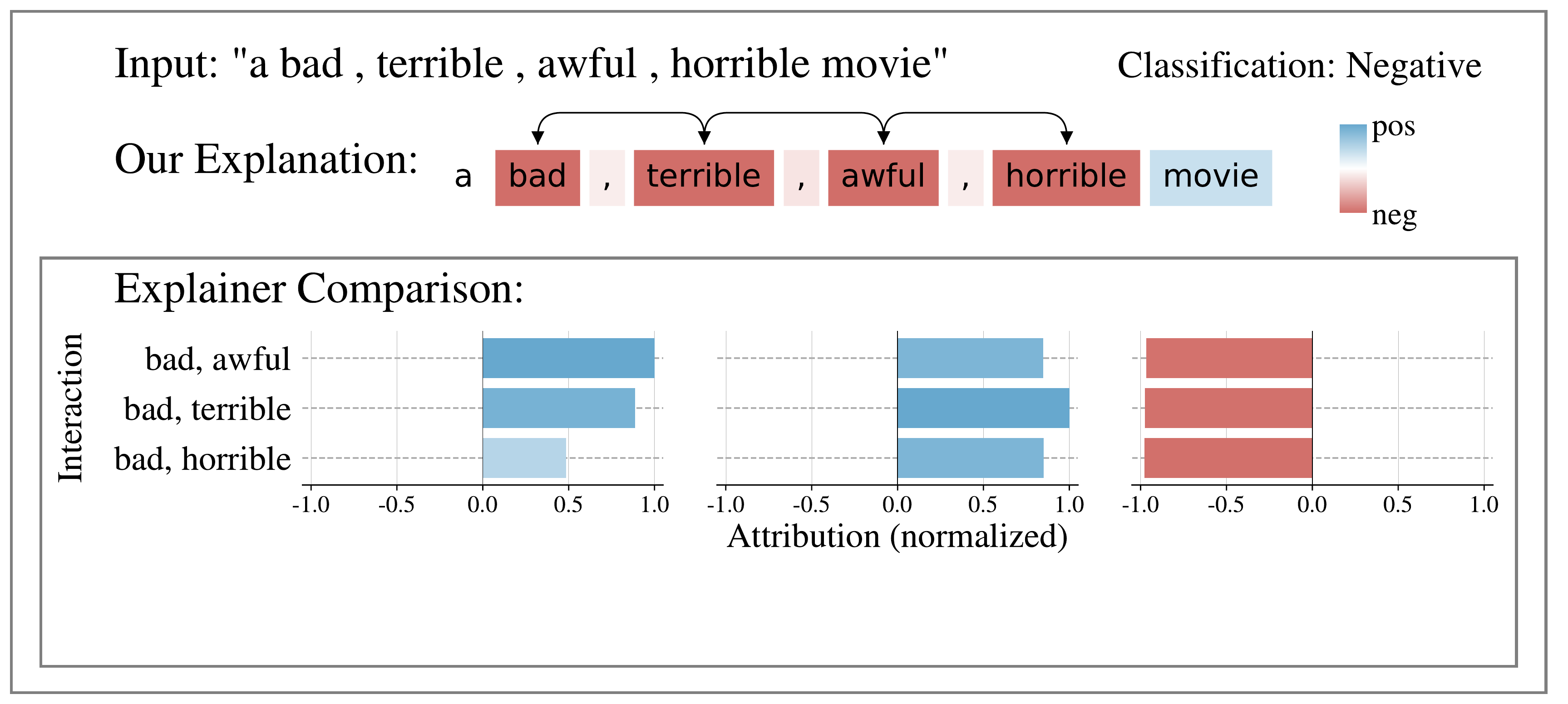}
};
\node[align=center]  at (3.9,0.66) {\footnotesize Integrated Hessians~\cite{janizek2020explaining}};
\node[align=center]  at (7.3,0.66) {\footnotesize\makecell{ Shapley Taylor\\ Interaction Index~\cite{dhamdhere2019shapley}}};
\node[align=center]  at (10.5,0.66) {\footnotesize Our Method};
\end{tikzpicture}}
\vspace{-0.07in}

\caption{Our explanation for the sentiment analysis example of~\cite{janizek2020explaining}. Colors indicate sentiment, and arrows indicate interactions. Compared to other axiomatic interaction explainers, only our work corroborates our intuition by showing negative attribution among top-ranked interactions.\label{fig:motiv}}
\vspace{-0.15in}
\end{figure}
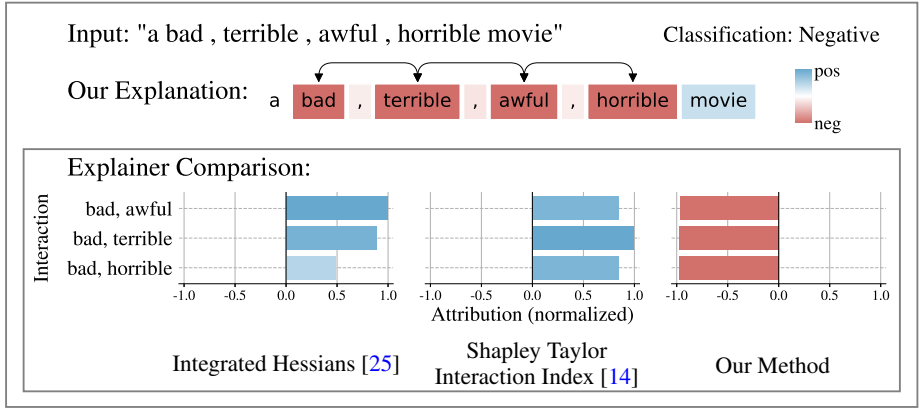

\vspace{-0.06in}
\section{Notations and Background}
\vspace{-0.06in}
\label{sec:definition}
\label{sec:preliminaries}
We first introduce preliminaries that serve as a basis for  our discussions.

\textbf{Notations:} We use boldface lowercase symbols, such as $\V{x}$, to represent vectors. The $i$-th entry of a vector $\V{x}$ is denoted by $x_i$. For a set $\mathcal{S}$, its cardinality is denoted by $\abs{\mathcal{S}}$, and the operation $\setminus \mathcal{S}$ means all except $\mathcal{S}$. 
For $p$ features 
 in a dataset,
 let $\II$ be a subset of feature indices: $\II\subseteq\{1,2,\dots,p\}$. 
 For a vector $\V{x}\in \real^p$, 
  let $\V{x}_{\II}\in \real^{p}$ be defined element-wise in~\eqref{eq:xi}. In our discussions, a \emph{context} means $\V{x}_{\setminus\II}$.
  
  {
\setlength{\abovedisplayskip}{-14pt}
\setlength{\belowdisplayskip}{-8pt}
   \begin{align}
    (\V{x}_{\II})_i = \left\{\begin{array}{lr}
        x_i, & \text{if } i\in \II\\
        0 & \text{otherwise }\\
        \end{array}\right.
        \label{eq:xi}
     \end{align}
    }

\textbf{Problem Setup:} Let $f$ denote a  black-box model with scalar output. For multi-class classification, $f$ is assumed to be a class logit. We use a target vector $\V{x^{\star}}\in \real^p$ to denote the data instance where we wish to explain $f$, and $\V{x}'\in \real^p$ to denote a \emph{neutral baseline}. Here, the baseline is  a  reference vector
for $\V{x^{\star}}$ and conveys an  ``absence of signal'' as per~\cite{sundararajan2017axiomatic}.
These vectors form the space of $\mathcal{X}\subset \real^{p}$, where each element comes from either $x_i^{\star}$ or $x_i'$, i.e. $\mathcal{X} = \{(x_1, \dots, x_p) \mid x_i \in \{x^{\star}_i, x'_i\}, \forall i = 1,\dots, p \}$.   

\textbf{Feature Interaction:} 
The definition of the feature interaction of interest is formalized as follows.

\begin{restatable}[Statistical Non-Additive Interaction]{definition}{interaction}\label{def:interaction}
A function $f$ contains a statistical non-additive interaction of multiple features indexed in set $\II$
if and only if $f$ cannot be decomposed into a sum of $|\II|$ subfunctions $f_i$ , each excluding the $i$-th interaction variable:
    $f(\V{x}) \neq \sum_{i\in\II} f_i(\V{x}_{\setminus \{i\}})$.
\vspace{-0.035in}
\end{restatable}
Def.~\ref{def:interaction} identifies a non-additive effect among all features $\mathcal{I}$ on the output of function $f$~\cite{friedman2008predictive, sorokina2008detecting, tsang2017detecting}.
For example, this means that 
the function $\text{ReLU}(x_1+x_2)$
creates a feature interaction
because it cannot be represented as an 
addition of univariate functions, i.e., 
 $\text{ReLU}(x_1+x_2)\neq f_1(x_2)$ + $f_2(x_1)$ (Fig.~\ref{fig:relu}). 
 We refer to individual feature effects which do not interact with other features as \emph{main effect}. Higher-order feature interactions are captured by $\abs{\II}>2$, i.e. interactions larger than pairs. Additionally, if a higher-order interaction exists,  all of its subsets also exist as interactions~\cite{sorokina2008detecting, tsang2017detecting}.

\vspace{-0.05in}
\section{{\framework} Interaction Attribution}
\vspace{-0.05in}

\label{sec:attribution}

We begin by presenting our feature attribution measure. Our feature attribution analyzes and assigns scores to detected feature interactions.
Our corresponding interaction detector is presented in \S\ref{sec:detection}.

\vspace{-0.055in}
\subsection{{\archattribute}}
\vspace{-0.055in}
\label{sec:proposed}

Let $\II$ be the set of feature indices that correspond to a desired attribution score. 
Our proposed attribution measure, called {\archattribute}, is given by 
{
\setlength{\abovedisplayskip}{8.5pt}
\setlength{\belowdisplayskip}{-7.5pt}
\begin{align}
    \phi(\II) = f(\V{x}^{\star}_\II+\V{x}'_{\setminus \II}) - f(\V{x}').
    \label{eq:att}
\end{align}}

{\archattribute} essentially  isolates the attribution of $\V{x}^{\star}_\II$ from the surrounding baseline context while also satisfying axioms (\S\ref{sec:axioms}).
We call this isolation an
``island effect'', where
the target features $\left\{x^{\star}_i\right\}_{i\in\II}$ do not specifically interact with the baseline features $\left\{x'_j\right\}_{j\in\mathcal{\setminus \II}}$.
For example, consider sentiment analysis on a phrase $\V{x}^{\star}=~$``not very bad'' with a baseline $\V{x}'=~$``\_ \_ \_'' . 
Suppose that we want to examine the attribution of an interaction $\II$ that corresponds to \{very, bad\} in isolation. In this case, the contextual word ``not'' also interacts with $\II$, which becomes apparent when small perturbations to the word ``not'' causes large changes to prediction probabilities. However, as we move further away from the word ``not'' towards the empty-word ``\_'' in the word-embedding space, small perturbations no longer result in large prediction changes, meaning that ``\_'' does not specifically interact with \{very, bad\}. This intuition motivates our use of the baseline context $\V{x}'_{\setminus\II}$ in~\eqref{eq:att}.

\begin{figure*}[t]
    \centering
    \begin{subfigure}[t]{0.33\textwidth}

    	\vskip 0pt
        \centering
          \includegraphics[scale=0.272,trim={0 0cm 1.35cm 0.25cm},clip]{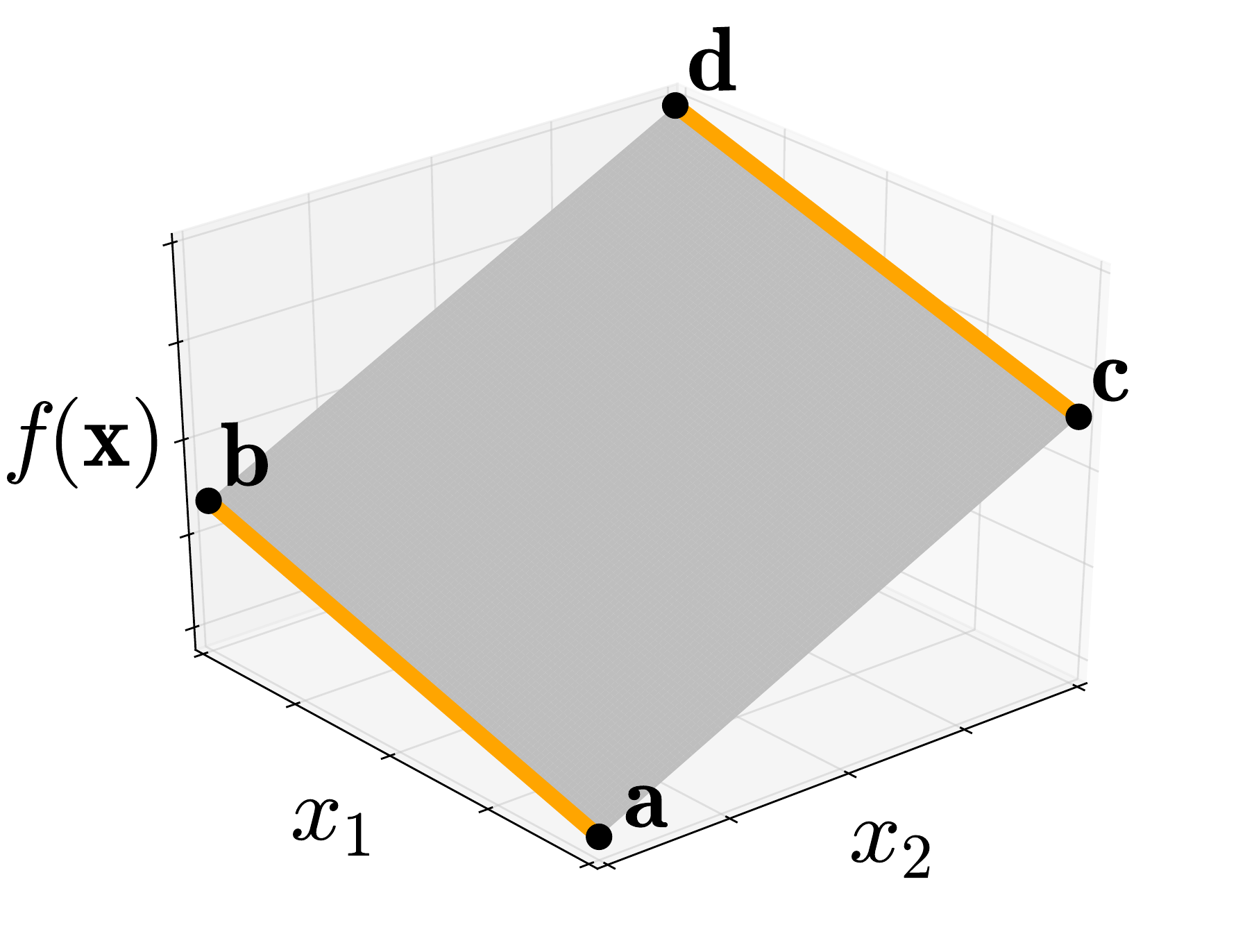}
        \vspace{-0.23in}
        \caption{Additive (linear) function\label{fig:plane}}
    \end{subfigure}%
    \begin{subfigure}[t]{0.33\textwidth}
     	\vskip 0pt
        \centering
        \includegraphics[scale=0.272,trim={0.4cm 0cm 1.35cm 0.25cm},clip]{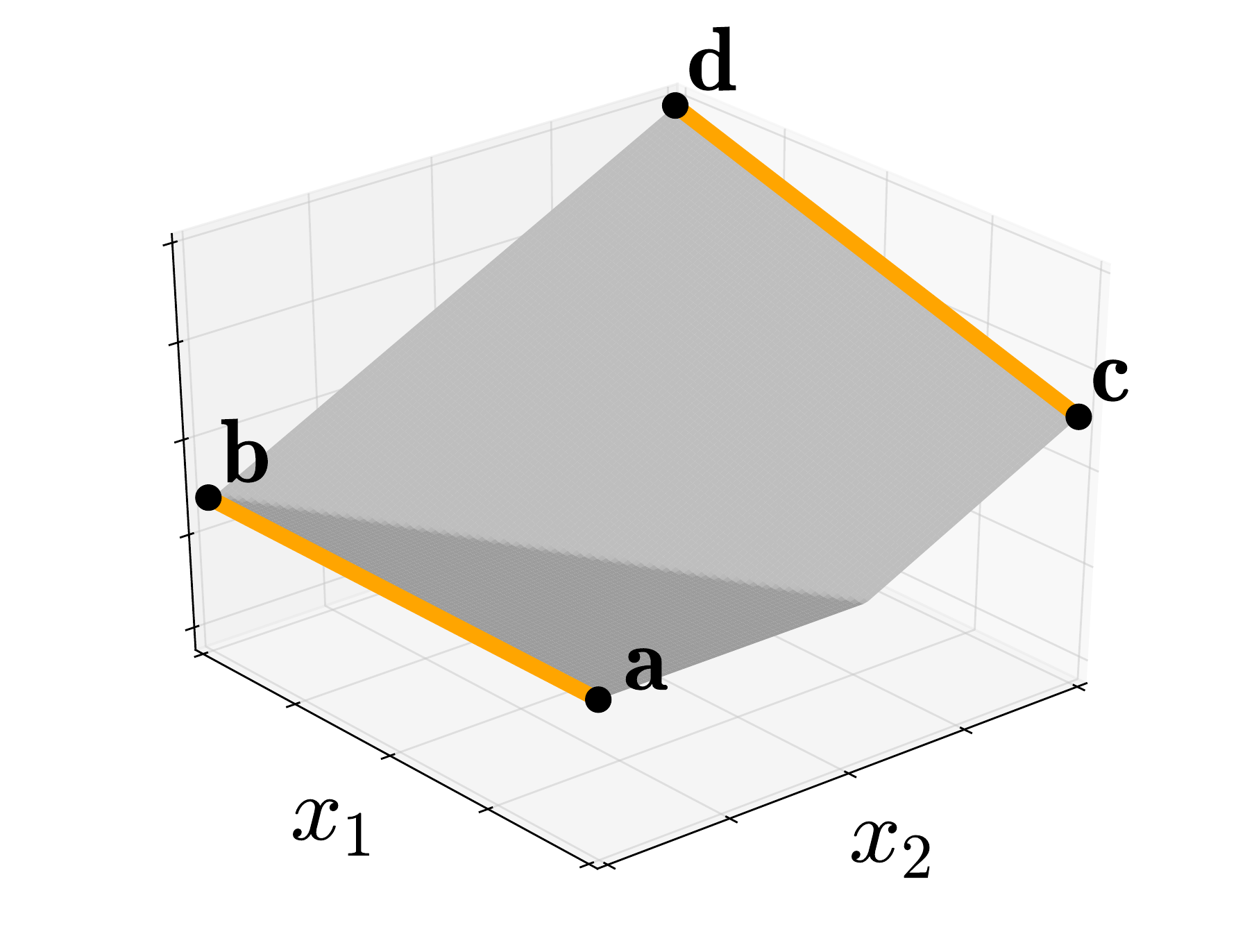}
        \vspace{-0.077in}
        \caption{Non-additive (ReLU) function\label{fig:relu}}
    \end{subfigure}
    \begin{subfigure}[t]{0.33\textwidth}
     	\vskip 0pt
        \centering
        \includegraphics[scale=0.272,trim={1.6cm 0cm 1.35cm 0.25cm},clip]{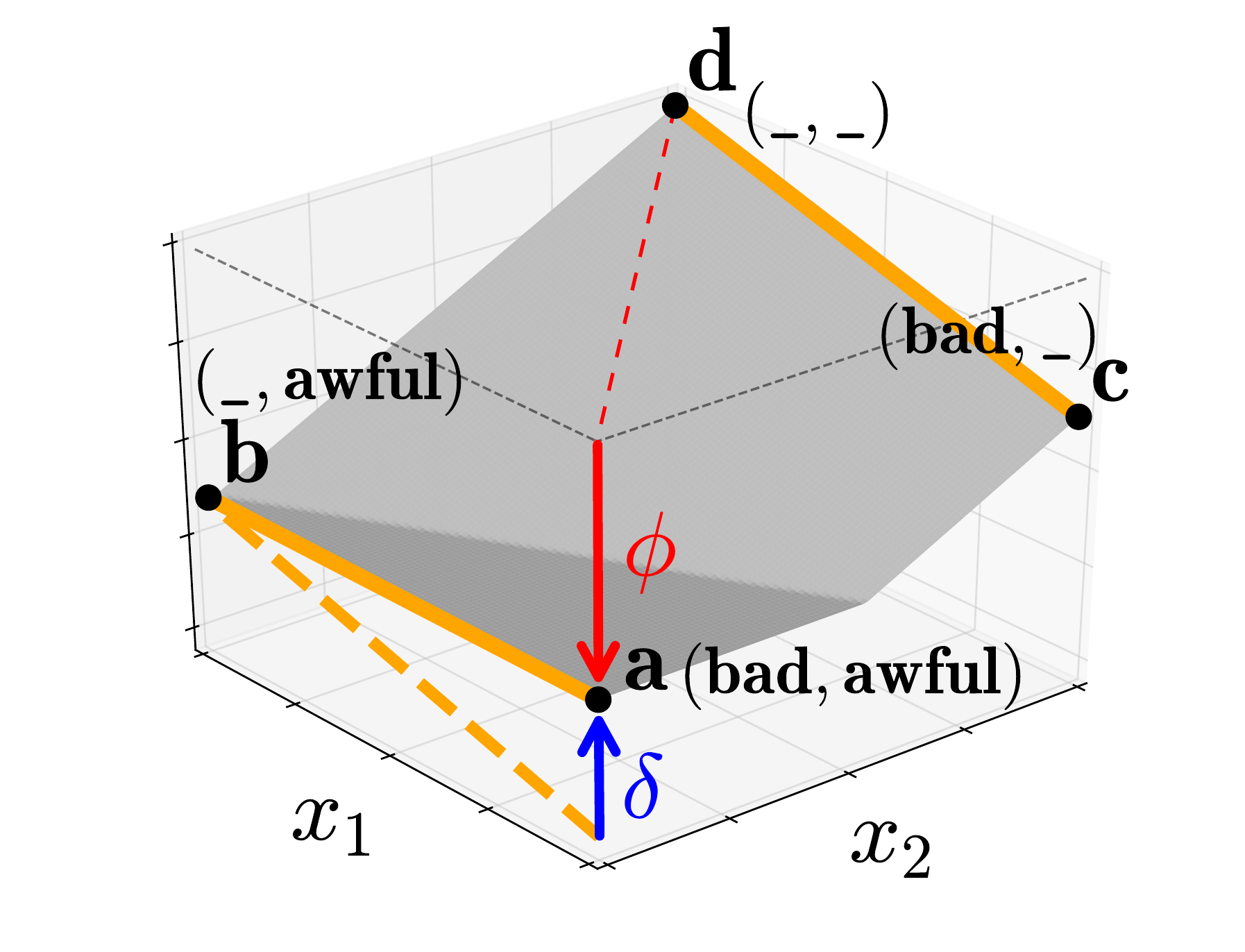}
        \vspace{-0.077in}
        \caption{$\delta$ vs. $\phi$ on a text example (Fig.~\ref{fig:motiv})\label{fig:relu_text}}
    \end{subfigure}
    \vspace{-0.005in}
    \caption{Non-additive interaction for $p=2$ features: The corner points are used to determine if $x_1$ and $x_2$ interact based on their non-additivity on $f$, i.e.
    they interact if
    $ \delta 	\propto (f(\V{a}) - f(\V{b})) - (f(\V{c}) - f(\V{d})) \neq 0$ (\S\ref{sec:discrete}).
    In (c), the attribution of (bad, awful) should be negative via $\phi$~\eqref{eq:att}, but Shapley Taylor Interaction Index uses the positive $\delta$. Note that $\phi$ depends on $\V{a}$ and $\V{d}$ whereas $\delta$ depends on $\V{a}$, $\V{b}$, $\V{c}$, and $\V{d}$. Also, Integrated Hessians is not relevant here since it does not apply to ReLU functions.
    \label{fig:geometric}
    \vspace{-0.105in}
    }
\end{figure*}

\vspace{-0.055in}
\subsection{Axioms}
\vspace{-0.055in}

\label{sec:axioms}

We now show how {\archattribute} obeys standard feature attribution axioms~\cite{sundararajan2017axiomatic}. Since {\archattribute} operates on feature sets, we generalize the notion of standard axioms to feature sets. To this end, we also propose a new axiom, Set Attribution, which allows us to work with feature sets.

Let $\mathcal{S} = \{{\II}_i\}_{i=1}^k$ be all $k$ feature interactions and main effects of $f$ in the space $\mathcal{X}$ (defined in \S\ref{sec:preliminaries}), where  we take the union of  overlapping sets in $\mathcal{S}$. Later in \S\ref{sec:detection}, we explain how to obtain $\mathcal{S}$. 

\textbf{Completeness:}
We consider a generalization of the completeness axiom for which the sum of all attributions equals $f(\V{x}^{\star})-f(\V{x}')$.
The axiom tells us how much feature(s) impact a prediction.

\begin{restatable}[Completeness on $\mathcal{S}$]{lemma}{complete}\label{thm:complete}
The sum of all attributions by {\archattribute} for the disjoint sets in $\mathcal{S}$ equals the difference of $f$ between $\V{x}^{\star}$ and the baseline $\V{x}'$: $f(\V{x}^{\star})-f(\V{x}')$.
\vspace{-0.018in}
\end{restatable}

The proof is in Appendix~\ref{apd:complete}. We can easily see {\archattribute} satisfying this axiom in the limiting case where $k=1$, $\II_1 = \{i\}_{i=1}^p$ because~\eqref{eq:att} directly becomes $f(\V{x}^{\star}) - f(\V{x}')$.
Existing interaction / group attribution methods: Sampling Contextual Decomposition (SCD)~\cite{jin2019towards}, its variant (CD)~\cite{murdoch2018beyond,singh2018hierarchical}, Sampling Occlusion (SOC)~\cite{jin2019towards}, and Shapley Interaction Index (SI)~\cite{grabisch1999axiomatic} do not satisfy completeness, whereas
 Integrated Hessians (IH)~\cite{janizek2020explaining} and Shapley Taylor Interaction Index (STI)~\cite{dhamdhere2019shapley} do.

\textbf{Set Attribution:}
We propose an axiom for interaction attribution called \textbf{Set Attribution} to work with feature sets as opposed to individual features and follow the additive structure of a  function.

\begin{restatable}[Set Attribution]{axiom}{attribution}\label{thm:attribution}
If $f:\real^p \rightarrow \real$ is a function in the form of 
        $f(\V{x}) = \sum_{i=1}^{k} \varphi_i(\V{x}_{\II_i}) $
where $\{\II_i\}_{i=1}^k$ are disjoint and functions $\{\varphi_i(\cdot)\}_{i=1}^k$ have roots, then an interaction attribution method admits an attribution  for feature set $\II_i$ as $\varphi_i(\V{x}_{\II_i})$ $\forall i=1,\dots,k$.
\vspace{-0.018in}
\end{restatable}

For example, if we consider a function $y=x_1x_2 + x_3$; it makes sense for the attribution of the $x_1x_2$ interaction to be the value of $x_1x_2$ and the attribution for the $x_3$ main effect to be the value of $x_3$.

\begin{restatable}[Set Attribution on $\mathcal{S}$]{lemma}{set}\label{thm:set}
For $\V{x}=\V{x}^{\star}$ and a  baseline $\V{x}'$ such that $\varphi_i(\V{x}'_{\II_i})=0~\forall i = 1,\dots,k$, {\archattribute} satisfies the  Set Attribution axiom and provides attribution $\varphi_i(\V{x}_{\II_i})$ for set $\II_i~\forall i$.
\vspace{-0.018in}
\end{restatable}

The proof is in Appendix~\ref{apd:set_attribution}, which follows from Lemma~\ref{thm:complete}. Neither SCD, CD, SOC, SI, IH, nor STI satisfy Set Attribution (shown in Appendix~\ref{apd:counter}). We can enable
Integrated Gradients (IG)~\cite{sundararajan2017axiomatic} to satisfy our axiom by summing its attributions within each  feature set of $\mathcal{S}$.
{\archattribute} differs from IG by its ``island effect'' (\S\ref{sec:proposed}) and model-agnostic properties.

\textbf{Other Axioms:}
{\archattribute} also satisfies the remaining axioms:
Sensitivity, Implementation Invariance, Linearity, and Symmetry-Preserving, which we show via Lemmas~\ref{thm:sensitivit_a}-\ref{thm:symmetry} in Appendix~\ref{apd:other}.

\textbf{Discussion:}
Several axioms required disjoint interaction and main effect sets in $\mathcal{S}$. Though interactions are not necessarily disjoint by definition (Def.~\ref{def:interaction}), it is reasonable to merge overlapping interactions to obtain compact visualizations, as shown in Fig.~\ref{fig:motiv} and later experiments (\S\ref{sec:exp_proposed}). The disjoint sets also allow {\archattribute} to yield identifiable non-additive  attributions in the sense that it can identify the attribution given a feature set in $\mathcal{S}$. This contrasts with  Model-Agnostic Hierarchical Explanations (MAHE)~\cite{tsang2018can}, which yields unidentifiable attributions~\cite{wood2017generalized}.

\vspace{-0.05in}
\section{{\framework} Interaction Detection}
\vspace{-0.05in}
\label{sec:detection}

Our axiomatic analysis of {\archattribute} relied on $\mathcal{S}$, which
contains
interaction sets of $f$ on the space $\mathcal{X}$
(defined in \S\ref{sec:preliminaries}).
To develop an interaction detection method that works in tandem with {\archattribute},
we draw inspiration from the discrete interpretation of mixed partial derivatives.

\vspace{-0.05in}
\subsection{Discrete Interpretation of Mixed Partial Derivatives}
\label{sec:discrete}
\vspace{-0.05in}

Consider the plots in Fig.~\ref{fig:geometric}, which consist of points $\V{a}$, $\V{b}$, $\V{c}$, and $\V{d}$ that each contain two features. From a top-down view of each plot, the points form the corners of a rectangle, whose side lengths are 
  $h_1 = \abs{a_1-b_1} = \abs{c_1-d_1}$ and 
  $h_2 = \abs{a_2-c_2} =  \abs{b_2-d_2}$.
  When $h_1$ and $h_2$ are small, the mixed partial derivative w.r.t variables $x_1$ and $x_2$ is computed as follows. First, $\tfrac{\partial f (\V{a}) }{\partial x_1}\approx \tfrac{1}{h_1}\left(f(\V{a}) - f(\V{b})\right) $  and $\tfrac{\partial f (\V{c}) }{\partial x_1}\approx \tfrac{1}{h_1}\left(f(\V{c}) - f(\V{d})\right)$. Similarly, the mixed partial derivative is approximated as:
\begin{align}
    \tfrac{\partial^2 f}{\partial x_1x_2} &\approx \tfrac{1}{h_2} \left(\tfrac{\partial f (\V{a}) }{\partial x_1} - \tfrac{\partial f (\V{c}) }{\partial x_1} \right)
    \approx \tfrac{1}{h_1 h_2} \left( (f(\V{a}) - f(\V{b})) - (f(\V{c}) - f(\V{d}))\right).
    \label{eq:mixed}
\end{align}
When $h_1$ and $h_2$ become large,~\eqref{eq:mixed} tells us if a plane can fit through all four points $\V{a}$,$\V{b}$,$\V{c}$, $\V{d}$ (Fig.~\ref{fig:plane}), which occurs when~\eqref{eq:mixed} is zero.  In this domain where $x_1$ and $x_2$ only take two possible values each, a plane in the linear form $f(\V{x})=w_1x_1 + w_2x_2 + b$ is functionally equivalent to all functions of the form $f(\V{x}) =  f_1(x_1) + f_2(x_2) + b$, so any deviation from the plane, e.g. Fig.~\ref{fig:relu}, becomes non-additive. Consequently, a \emph{non-zero} value of~\eqref{eq:mixed} identifies a non-additive interaction by the definition of statistical interaction (Def.~\ref{def:interaction}). What's more, the magnitude of~\eqref{eq:mixed} tells us the degree of deviation from the plane, or the degree of non-additivity. (Additional details in Appendix~\ref{apd:mixed})

\vspace{-0.05in}
\subsection{{\archdetect}}
\vspace{-0.05in}

\label{sec:proposeddetection}

Leveraging these insights about mixed partial derivatives, we now discuss the two components of our proposed interaction detection technique -- {\archdetect}.

\vspace{-0.05in}
\subsubsection{Handling Context:}\label{sec:contexts}
 As defined in \S\ref{sec:axioms} and \S\ref{sec:detection}, our problem is how to identify interactions of $p$ features in $\mathcal{X}$ for our target data instance $\V{x}^{\star}$ and baseline $\V{x}'$. If $p=2$, then we can almost directly use~\eqref{eq:mixed}, where $\V{a}=(x^{\star}_1, x^{\star}_2)$, $\V{b}=(x'_1, x^{\star}_2)$, $\V{c}=(x^{\star}_1, x'_2)$, and $\V{d}=(x'_1, x'_2)$. However if $p>2$,
all possible combinations of features in $\mathcal{X}$ would need to be examined to thoroughly identify just one pairwise interaction. To see this, we first rewrite~\eqref{eq:mixed} to accommodate $p$ features, and square the result to measure interaction strength and be consistent with previous interaction detectors~\cite{friedman2008predictive,gevrey2006two}. The interaction strength between features $i$ and $j$ for a context $\V{x}_{\setminus \{i,j\}}$ is then
defined as
{
\setlength{\abovedisplayskip}{-3.1pt}
\setlength{\belowdisplayskip}{-1.4pt}
\begin{dmath}
    \omega_{i,j}(\V{x})= \left(\tfrac{1}{h_i h_j} \left(f(\V{x}^{\star}_{\{i,j\}} +  \V{x}_{\setminus{\{i,j\}}}) - f(\V{x}'_{\{i\}} + \V{x}^{\star}_{\{j\}} +  \V{x}_{\setminus{\{i,j\}}}) - f(\V{x}^{\star}_{\{i\}} + \V{x}'_{\{j\}} +  \V{x}_{\setminus{\{i,j\}}}) + f(\V{x}'_{\{i,j\}} +   \V{x}_{\setminus{\{i,j\}}})\right)\right)^2,\label{eq:pairwise}
\end{dmath}
}
where $h_i= \abs{x^{\star}_i - x'_i}$ and $h_j= \abs{x^{\star}_j - x'_j}$. The thorough way to identify the $\{i,j\}$ feature interaction is given by     $\bar{\omega}_{i,j}=\mathbb{E}_{\V{x}\in \mathcal{X}}\left[\omega_{i,j}(\V{x})\right]$,
where each element of $\V{x}_{\setminus\{i,j\}}$ is 
Bernoulli ($0.5$). This expectation 
is intractable because $\mathcal{X}$ has an exponential search space,
so we propose the first component of {\archdetect} for efficient pairwise interaction detection:
{
\setlength{\abovedisplayskip}{-2.2pt}
\setlength{\belowdisplayskip}{-0.9pt}
\begin{dmath}
    \bar{\omega}_{i,j} =  \frac{1}{2} \left(\omega_{i,j}(\V{x^{\star}}) + \omega_{i,j}(\V{x}')\right).
    \label{eq:efficient}
\end{dmath}
}
Here, we estimate the expectation by leveraging the physical meaning of the interactions and {\archattribute}'s axioms via the different contexts of $\V{x}$ in \eqref{eq:efficient} as follows:
\begin{itemize}
    \item \textbf{Context of} $\V{x}^{\star}$\textbf{:} An important interaction is one due to multiple $\V{x}^{\star}$ features. As a concrete example, consider an image representation of a cat which acts as our target data instance. The following higher-order interaction,
$if~x_{ear}= x^{\star}_{ear}~and~x_{nose}= x^{\star}_{nose}~and~  x_{fur}= x^{\star}_{fur} ~then~f(\V{x})= {high~cat~probability}$, is responsible for classifying  ``cat''.
We can detect any pairwise subset $\{i,j\}$ of this interaction by setting the context as $\V{x}^{\star}_{\setminus \{i,j\}}$ using $\omega_{i,j}(\V{x}^{\star})$.

\item \textbf{Context of} $\V{x}'$\textbf{:} Next, we consider $\V{x}'_{\setminus\{i,j\}}$ to detect interactions via $\omega_{i,j}(\V{x}')$, which  helps us establish  {\archattribute's} completeness (Lemma \ref{thm:complete}). This also separates out  effects of any higher-order baseline interactions from $f(\V{x'})$ in~\eqref{eq:separation} (Appendix~\ref{apd:complete}) and recombine their effects in~\eqref{eq:sum_separate}. From an interpretability standpoint, the $\V{x}'_{\setminus \{i,j\}}$ context ranks pairwise interactions w.r.t. a standard baseline. This context is also used by {\archattribute}~\eqref{eq:att}.

\item \textbf{Other Contexts:} The first two contexts accounted for any-order interactions created by either target or baseline features and a few interactions created by a mix of baseline and target features. The remaining interactions
specifically require
a mix of  $>3$ target and baseline features. This case is unlikely and is excluded, as we discuss next.
\end{itemize}

The following assumption formalizes our intuition for the \emph{Other Contexts} setting where there is a mix of higher-order ($>3$) target and baseline feature interactions.

\begin{restatable}[Higher-Order Mixed-Interaction]{assumption}{island}
\vspace{0.035in}
\label{assum:island}
For any feature set $\II$ where $\abs{\II}>3$ and any pair of non-empty disjoint sets $\mathcal{A}$ and $\mathcal{B}$ where $\mathcal{A}\cup\mathcal{B} = \II$, the  instances $\V{x}\in \mathcal{X}$ such that $x_i = x^{\star}_i~\forall i\in\mathcal{A}$ and $x_j = x'_j~\forall j\in \mathcal{B}$ 
do not  cause a higher-order interaction of all features   $\{x_k \}_{k\in\mathcal{\II}}$ via $f$.
\vspace{-0.025in}
\end{restatable}

Assumption~\ref{assum:island} has a similar intuition as {\archattribute} in~\S\ref{sec:proposed} that target features do not specifically interact with baseline features. To understand this assumption, 
consider the original sentiment analysis example in Fig.~\ref{fig:motiv} simplified as $\V{x}^{\star}=\text{``bad terrible awful horrible movie''}$ where $\V{x}' =~$``\_ \_ \_ \_ \_''. It is reasonable to assume that there is no special interaction created by token sets such as
\{bad, terrible, \_ , horrible\} 
or \{\_ , \_ , \_ , horrible\} due to the meaningless nature of the ``\_'' token.

 \textbf{Efficiency:} In~\eqref{eq:efficient}, {\archdetect} attains interaction detection over all pairs $\{i,j\}$ in $\mathcal{O}(p^2)$  calls of $f$. Note that in~\eqref{eq:pairwise}, most function calls are reusable during pairwise interaction detection. 

\vspace{-0.13in}

\subsubsection{Detecting Disjoint Interaction Sets:}\label{sec:disjoint}
In this section, the aim here is to recover arbitrary size and disjoint non-additive feature sets $\mathcal{S} = \{{\II}_i\}$ (not just pairs). 
 {\archdetect} looks at the union of overlapping pairwise interactions to obtain disjoint feature sets. Merging these pairwise interactions captures any existing higher-order interactions automatically since the existence of a higher-order interaction automatically means all its subset interactions exist~(\S\ref{sec:definition}).
In addition, {\archdetect} merges these overlapped pairwise interactions with all individual feature effects to account for all features. The time complexity of this merging process is also  $\mathcal{O}(p^2)$.

\section{Experiments}
\vspace{-0.1in}
\label{sec:experiments}

\subsection{Setup} 
\vspace{-0.05in}

\label{sec:setup}
We  conduct experiments first on {\archdetect} in \S\ref{sec:exp_detector} then on {\archattribute} in~\S\ref{sec:exp_proposed}. We then visualize their combined form as {\framework} in~\S\ref{sec:exp_proposed}. 
Throughout our experiments, we commonly study {\bert}~\cite{devlin2019bert, Wolf2019HuggingFacesTS} on text-based sentiment analysis and {\resnet}~\cite{he2016deep} on image classification. {\bert} was fine-tuned on the SST dataset~\cite{socher2013recursive}, and {\resnet} was pretrained on {\imagenet}~\cite{imagenet_cvpr09}.

For sentiment analysis, we set the baseline vector $\V{x}'$ to be the tokens ``\_'',  in place of each word-token from $\V{x}^{\star}$. For image classification, we set $\V{x}'$ to be an all-zero image, and use the Quickshift superpixel segmenter~\cite{vedaldi2008quick} as per the need for input dimensionality reduction~\cite{tsang2020feature} (details in Appendix~\ref{apd:reduction}). We set $h_1=h_2=1$ for both domains. Several methods we compare to are common across experiments, in particular IG, IH, (disjoint) MAHE, SI, STI, and Difference, defined as $\phi_{d}(\II) =  f(\V{x}^{\star}) - f(\V{x}'_{\II}+ \V{x}^{\star}_{\setminus \II})$. 

\vspace{-0.05in}
\subsection{{\archdetect}}
\label{sec:exp_detector}
\vspace{-0.05in}

\begin{figure}[t]
    \begin{minipage}{.6\textwidth}
    \captionof{table}{Comparison of interaction detectors (b) on synthetic ground truth in (a).} 
    \vspace{-0.03in}
    \label{table:detection_eval}
        \begin{subtable}{1\columnwidth}
        
         \caption{Functions with Ground Truth Interactions\label{table:synthetic}}

        \vspace{-0.065in}

        \resizebox{0.95\columnwidth}{!}{%
        
        \begin{tabular}{lc}
        
        \toprule   
        $F_1(\V{x})=$ &  $\sum_{i=1}^{10}\sum_{j=1}^{10} x_ix_j + \sum_{i=11}^{20}\sum_{j=21}^{30} x_ix_j+  \sum^{40}_{k=1}{x}_k$\\
        $F_2(\V{x})=$ &  $\bigwedge(\V{x}; \{x_i^{\star}\}_{i=1}^{20})$  + $\bigwedge(\V{x} ;\{x_i^{\star}\}_{i=11}^{30}) + \sum_{j=1}^{40}x_j$\\
        $F_3(\V{x})=$ & $\bigwedge(\V{x}; \{x_i'\}_{i=1}^{20})$  + $\bigwedge(\V{x} ;\{x_i^{\star}\}_{i=11}^{30}) + \sum_{j=1}^{40}x_j$\\
        $F_4(\V{x})=$ &  $\bigwedge(\V{x}; \{x_1^{\star}, x^{\star}_2\} \cup  \{x_3'\})$  + $\bigwedge(\V{x} ;\{x_i^{\star}\}_{i=11}^{30}) + \sum_{j=1}^{40}x_j$\\

        \bottomrule
        \end{tabular}

           }

        \end{subtable}
        \vspace{0.115in}

        \begin{subtable}{1\columnwidth}
        \caption{Pairwise Interaction Ranking AUC. The baseline methods fail to detect interactions suited for the desired contexts in \S\ref{sec:contexts}. \label{table:detection}}

        \vspace{-0.065in}

        \resizebox{0.95\columnwidth}{!}{%

            \begin{tabular}{lcccc}
            \toprule   
            Method & $F_1$ & $F_2$ & $F_3$ & $F_4$\\
            \midrule
            Two-way ANOVA & $1.0$ & $0.51$& $0.51$ &$0.55$  \\
            Integrated Hessians & $1.0$ & N/A & N/A & N/A\\
            Neural Interaction Detection& $0.94$ & $0.52$ & $0.48$ & $0.56$\\
            Shapley Interaction Index& $1.0$ & $0.50$ & $0.50$ & $0.51$\\
            Shapley Taylor Interaction Index&  $1.0$ & $0.50$ & $0.53$ & $0.51$\\
            {\archdetect}~(this work) & $1.0$ & $1.0$ & $1.0$ & $1.0$\\
            \bottomrule
        \end{tabular}

           }

        \end{subtable}

    \end{minipage}%
    \hfill
    \begin{minipage}{0.38\textwidth}

   \begin{subfigure}[b]{\textwidth}
        \centering
       \includegraphics[scale=0.425]{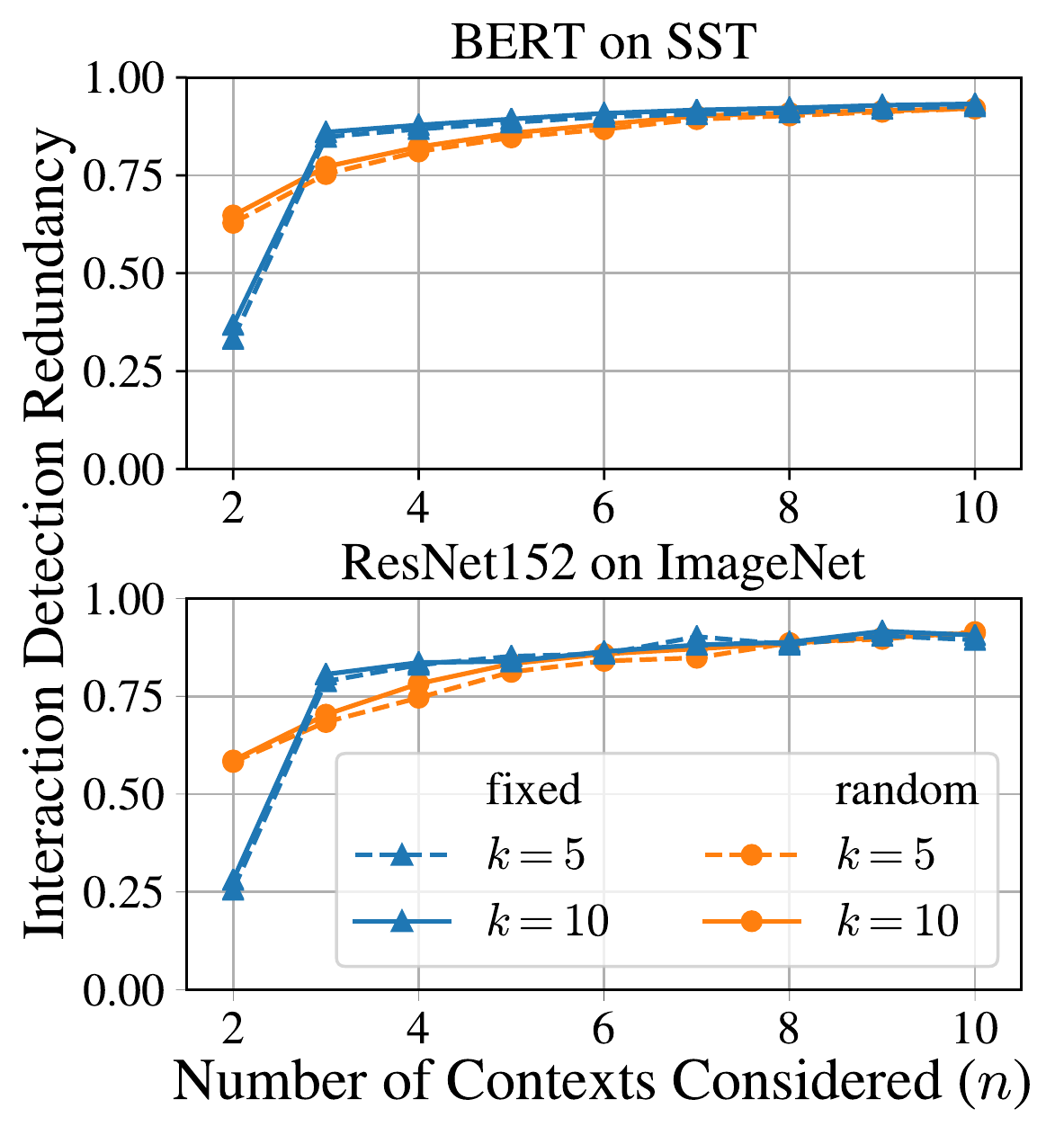}
    \end{subfigure}%
    \vspace{-0.05in}
     \caption{Interaction detection overlap (redundancy) with added contexts to~\eqref{eq:efficient}. ``fixed'' at $n=2$ (\archdetect) already shows good stability.   }
    \label{fig:redundancy}
        
    \end{minipage}
    \vspace{-0.11in}
\end{figure}

We validate {\archdetect}'s 
performance
via synthetic ground truth and redundancy experiments.

\textbf{Synthetic Validation:}
We set
$\V{x}^{\star} = [1,1,\dots,1]\in\real^{40}$ and
$\V{x}' = [-1,-1,\dots,-1]\in\real^{40}$. Let $z[\cdot]$ be a key-value pair function such that  $z[i] = x_i$ for key $i\in z.keys$ and value $x_i$, so we can define 
 \[
\bigwedge(\V{x}; z) \vcentcolon=    
    \left\{\begin{array}{lr}
        ~~1, & \text{if } x_i = z[i] ~\forall i \in z.keys\\
        -1 & \text{for all other cases}.\\
        \end{array}\right.
  \]
Table~\ref{table:synthetic} shows functions with ground truth interactions
suited for the  desired contexts in \S\ref{sec:contexts}.
Table~\ref{table:detection} shows  interaction detection AUC on these functions by {\archdetect}, IH, SI, STI, Two-way ANOVA~\cite{fisher1925statistical} and the state-of-the-art Neural Interaction Detection~\cite{tsang2017detecting}. On 
$F_2$, $F_3$, \& 
$F_4$, the baseline methods fail  because they are not designed to detect the interactions of our desired contexts (\S\ref{sec:contexts}).

\textbf{Interaction Redundancy:} The purpose of the next experiments is to see if {\archdetect} can omit 
certain 
higher-order interactions. We study the form of~\eqref{eq:efficient} by examining the redundancy of interactions as new contexts are added to~\eqref{eq:efficient}, which we now write as $   \bar{\omega}_{i,j}(C) =  \frac{1}{C} \sum_{c=1}^C \omega_{i,j}(\V{x}_c)$. Let $n$ be the number of contexts considered, and $k$ be the number of top pairwise interactions selected after running pairwise interaction detection via $\bar{\omega}_{i,j}$ for all $\{i,j\}$ pairs. Interaction redundancy is the overlap ratio of two sets of top-$k$ pairwise interactions, one generated via $\bar{\omega}_{i,j}(n)$ and the other one via $\bar{\omega}_{i,j}(n-1)$ for some integer $n\geq2$. We generally 
expect the redundancy to increase as $n$ increases, which we initially observe in Fig.~\ref{fig:redundancy}. 
Here,
``fixed'' and ``random'' correspond to different context sequences $\V{x}_1, \V{x}_2, \dots, \V{x}_N$.
The ``random'' sequence uses random samples from $\mathcal{X}$ for all $\{\V{x}_i\}_{i=1}^N$, whereas the ``fixed'' sequence is fixed in the sense that 
$\V{x}_1=\V{x}^{\star}$, $\V{x}_2=\V{x}'$, and the remaining $\{\V{x}_i\}_{i=3}^N$ are random samples.
Experiments are done on the {\sst} test set for {\bert} and $100$ random test images in {\imagenet} for {\resnet}.  Notably, the ``fixed'' setting has very low redundancy at $n=2$ ({\archdetect}) versus ``random''. As soon as $n=3$, the redundancy jumps and stabilizes quickly. These experiments support Assumption~\ref{assum:island} and~\eqref{eq:efficient}
to omit specified higher-order interactions.

\vspace{-0.05in}
\subsection{{\archattribute} \& {\framework}}
\vspace{-0.05in}
\label{sec:exp_proposed}

\begin{table}[t]
      \centering
        \caption{Comparison of attribution methods on {\bert} for sentiment analysis and {\resnet} for image classification. Performance is measured by the correlation ($\rho$) or AUC of the top and bottom $10\%$ of attributions for each method with respect to reference scores defined in~\S\ref{sec:exp_proposed}.
        }
     \resizebox{0.88\columnwidth}{!}{%

    \begin{tabular}{lccc}
    \toprule
    \multirow{2}{*}{Method}&\multicolumn{2}{c}{\makecell{\bert \\ Sentiment Analysis}}&\multicolumn{1}{c}{\makecell{\resnet\\Image Classification}}
    \\ \cmidrule(lr){2-3}\cmidrule(lr){4-4}
    
    &
    Word $\rho$ & Phrase $\rho$ $\dagger$    &
     Segment AUC $\dagger$
    \\
        
    \midrule 
    Difference  & $0.427$  & $0.639$ & $0.705$   \\
    Integrated Gradients (IG) & $0.568$  & $0.737$ & $0.786$   \\
    Integrated Hessians (IH) & N/A  & $0.128$ & N/A  \\
    Model-Agnostic Hierarchical Explanations (MAHE)   &  $0.673$ & $0.702$ & $0.712$ \\
    Shapley Interaction Index (SI) & $0.168$  & $-0.018$ &  $0.530$ \\
    Shapley Taylor Interaction Index (STI) &$0.754$  & $0.286$ & $0.626$  \\
    *Sampling Contextual Decomposition (SCD)  & $0.709$  & $0.742$ & N/A  \\
    *Sampling Occlusion (SOC)  & $0.768$  & $0.794$ & N/A  \\
    {\archattribute} (this work) & $\mathbf{0.809}$ & $\mathbf{0.836}$ &  $\mathbf{0.919}$\\
    \bottomrule
    \multicolumn{4}{r}{
      \begin{minipage}{10.3cm}
          \vspace{0.05in}

        \scriptsize $\dagger$ Methods that cannot tractably run for arbitrary feature set sizes are only run for pairwise feature sets.\\
        \scriptsize  * SCD and SOC are specifically for sequence models and contiguous words.
      \end{minipage}
    }
    \end{tabular}
    
    }
    \label{table:eval}
    \vspace{-0.13in}
\end{table}

\begin{figure}
    \centering
    \includegraphics[scale=0.323,trim={0 3.5cm 28.5cm 0.5cm},clip]{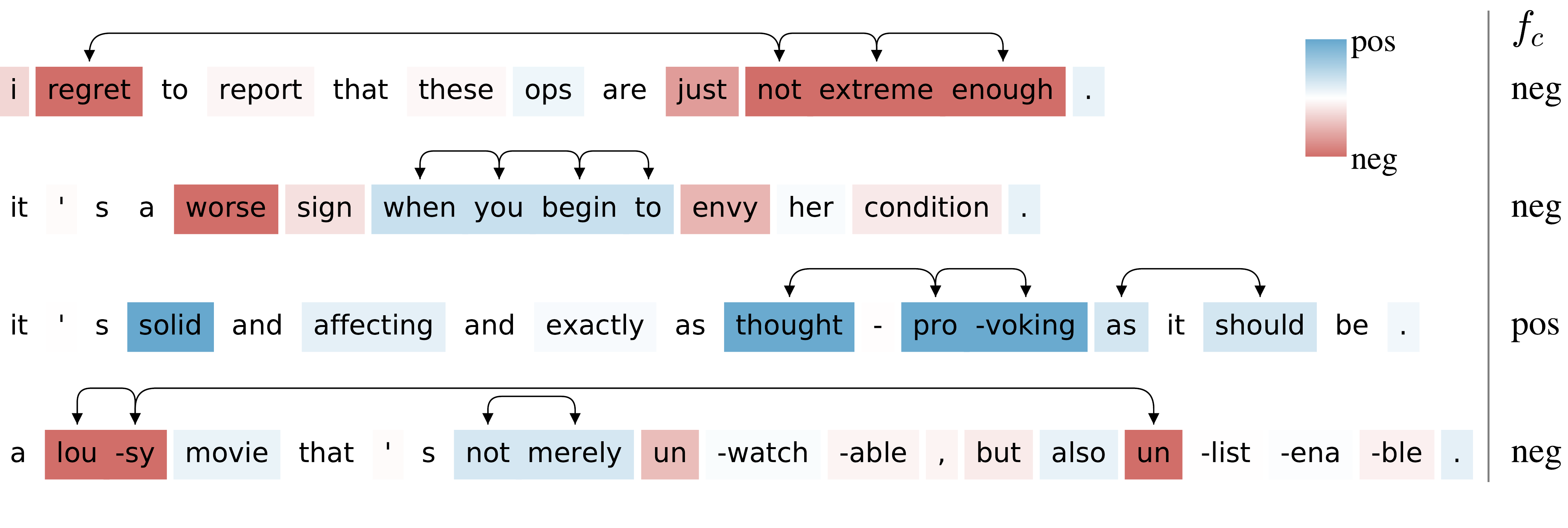}

    \vspace{-0.1in}
    \caption{Our {\bert} visualizations on random test sentences from {\sst} under {\bert} tokenization. Arrows indicate interactions, and colors indicate attribution strength. $f_c$ is the sentiment classification. The interactions point to salient and sometimes long-range sets of words, and the colors are sensible. }
    \label{fig:text_exp} %
    \vspace{-0.13in}
\end{figure}

We study the interpetability of {\archattribute} by comparing its attribution scores to ground truth annotation labels on subsets of features. For fair comparison, we look at extreme attributions (top and bottom $10\%$) for each baseline method.  We then visualize the combined {\framework} framework. Additional comparisons on attributions, runtime, and visualizations are shown in Appendices~\ref{apd:quantile},~\ref{apd:runtime},~\ref{apd:viz}.

\textbf{Sentiment Analysis:}
For this task, we compare {\archattribute} to other explanation methods on two metrics: phrase correlation  (Phrase $\rho$) and word correlation (Word $\rho$) on the {\sst} test set (metrics are from~\cite{jin2019towards}). Phrase $\rho$ is the Pearson correlation between estimated phrase attributions and {\sst} phrase labels (excluding prediction labels) on a $5$-point sentiment scale. Word $\rho$ is unlike our label-based evaluations
by computing
the Pearson correlation between estimated word attributions and the corresponding coefficients of a global bag-of-words linear model, which is also trained on the {\sst} dataset. 
In addition to the aforementioned baseline methods in~\S\ref{sec:setup}, we include the state-of-the-art SCD and SOC methods for sequence models~\cite{jin2019towards} in our evaluation. In Table~\ref{table:eval}, {\archattribute}
compares favorably to all methods where we consider the top and bottom $10\%$ of the attribution scores for each method. We obtain similar performance  across all other percentiles in Appendix~\ref{apd:quantile}.

We visualize {\framework} explanations on $\mathcal{S}$ generated by top-$3$ pairwise interactions (\S\ref{sec:disjoint})  in Fig.~\ref{fig:text_exp}. The sentence examples are randomly selected from the {\sst} test set. The visualizations show interactions and individual feature effects which all have reasonable polarity and intensity. Interestingly, some of the interactions, e.g. between ``lou-sy'' and ``un'', are long range.

\begin{figure}[t]
    \centering
    \includegraphics[scale=0.2,trim={0cm 3cm 0cm 0cm}]{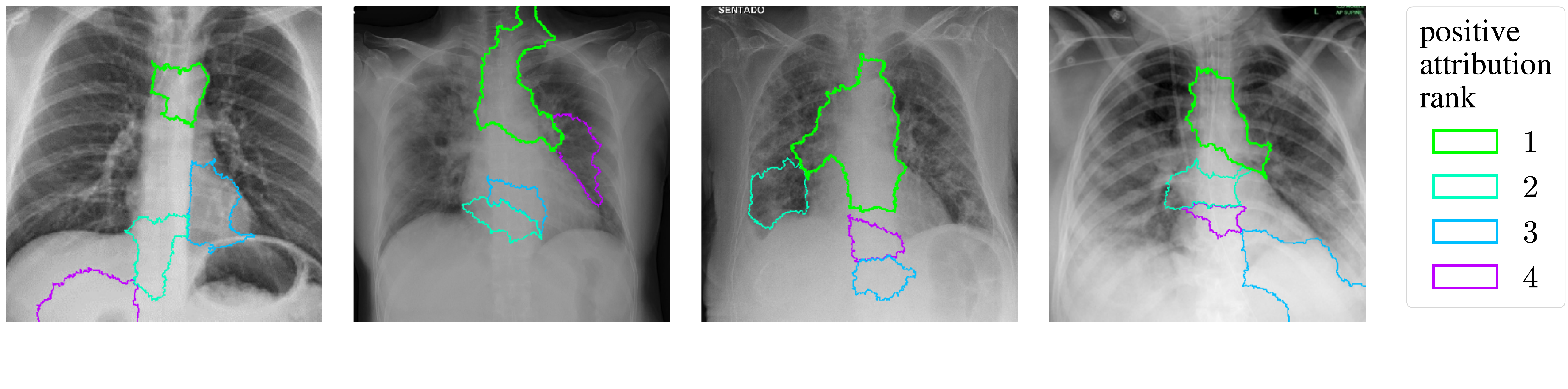}

    \vspace{-0.0475in}
    \caption{Our explanations of a COVID-19  classifier
    (COVID-Net)~\cite{wang2020covid} on randomly selected test X-rays~\cite{chowdhury2020can,cohen2020covid} classified as COVID positive. COVID-Net accurately distinguishes COVID from pneumonia and normal X-rays. Colored outlines indicate detected feature sets with positive attribution. 
     The explanations tend to detect on the ``great vessels'' outlined in green, which are mostly interactions.}
    
    \label{fig:img_exp}
    \vspace{-0.06in}
\end{figure}    

\textbf{Image Classification:}
On image classification, we compare {\archattribute} to relevant baseline methods on a ``Segment AUC'' metric, which computes the agreement between the estimated attribution of an image segment and that segment's label. We obtain segment labels from the MS COCO dataset~\cite{lin2014microsoft} and match them to the label space of {\imagenet}. All explanation attributions are computed relative to {\resnet}'s top-classification in the joint label space. The segment label thus becomes whether or not the segment belongs to the same class as the top-classification. Evaluation is conducted on all segments with valid labels in the MS COCO dev set. {\archattribute} performs especially well on extreme attributions in Table~\ref{table:eval}, as well as all attributions (in Appendix~\ref{apd:quantile}).

Fig.~\ref{fig:img_exp} visualizes
{\framework} on
an accurate COVID-$19$ classifier for chest X-rays~\cite{wang2020covid}, where  $\mathcal{S}$ is generated by top-$5$ pairwise interactions (\S\ref{sec:disjoint}).
Shown is a random selection of test X-rays~\cite{chowdhury2020can,cohen2020covid} that are classified COVID-positive. The explanations tend to detect the ``great vessels'' near the heart. 

\begin{wrapfigure}[14]{R}{0.34\textwidth}
 \center
    \vspace{-0.255in}
  \includegraphics[scale=0.5]{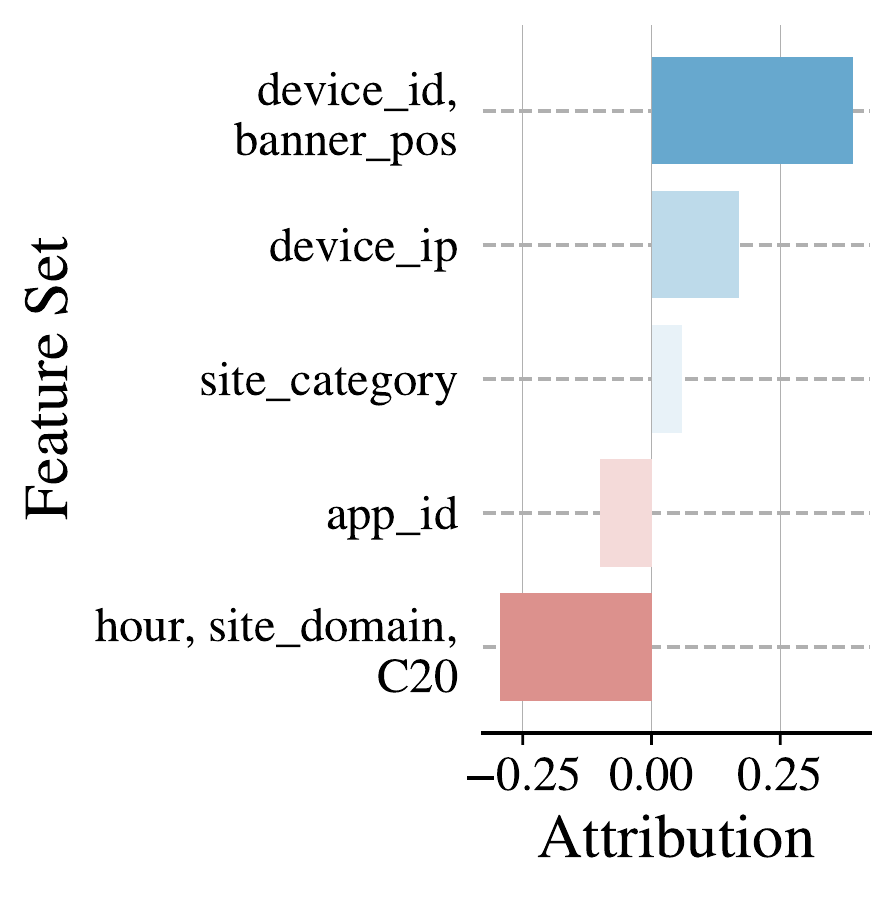}
    \vspace{-0.095in}
  \caption{Online ad-targeting: ``banner\_pos'' is used to target ads to a user per their ``device\_id''.}
  \label{fig:rec_exp}  
\end{wrapfigure}

\textbf{Recommendation Task:}
 Fig.~\ref{fig:rec_exp} shows {\framework}'s result for this task using a state-of-the-art AutoInt model~\cite{song2018autoint} for ad-recommendation.
Here, our approach finds a positive interaction between``device\_id'' and ``banner\_pos'' in the Avazu dataset~\cite{avazu}, meaning that the online advertisement model decides the banner position based on user device\_id. Note that for this task, there are no ground truth annotations.

\vspace{-0.05in}
\section{Related Works}
\vspace{-0.05in}
\label{sec:related}

\textbf{Attribution:} Individual feature attribution methods distill any  interactions of a data instance as attribution  scores for each feature. Many methods require the scores to sum to equal the output~\cite{shrikumar2017learning,binder2016layer,sundararajan2017axiomatic,lundberg2017unified,ribeiro2016should}, such as LIME and SHAP, which train surrogate linear explainer models on feature perturbations, and IG which invokes the fundamental theorem of calculus. Other methods compute attributions from an information theoretic perspective~\cite{chen2018learning} or strictly from model gradients~\cite{simonyan2013deep,ancona2018towards,selvaraju2017grad}. These methods interpret feature importance but not feature interactions.

\textbf{Feature Interaction:} Feature interaction explanation methods tend to either perform interaction detection~\cite{fisher1925statistical,sorokina2008detecting,tsang2017detecting,bien2013lasso,friedman2008predictive,gevrey2006two,ai2003interaction} or combined interaction detection and attribution~\cite{purushotham2014factorized,dhamdhere2019shapley,janizek2020explaining,lou2013accurate,tsang2018can,lundberg2018consistent}.
Relevant black-box interaction explainers are  STI~\cite{dhamdhere2019shapley} which uses random feature orderings to identify contexts for a variant of~\eqref{eq:pairwise} so that interaction scores satisfy completeness,  IH~\cite{janizek2020explaining} which extends IG with path integration for hessian computations, and 
MAHE~\cite{tsang2018can}, which trains surrogate explainer models for interaction detection and attribution.
STI and IH are axiomatic and satisfy completeness but their attributions are uninterpretable (Table~\ref{table:eval}) and inefficient. MAHE's attributions are  unidentifiable by training additive attribution models on overlapping feature sets. Several methods compute attributions on feature sequences or sets, such as SOC~\cite{jin2019towards}, SCD~\cite{jin2019towards}, and CD~\cite{murdoch2018beyond,singh2018hierarchical}, but they do not obey basic axioms. Finally, many methods are not model-agnostic, such as SCD, CD, IG, IH, GA2M~\cite{lou2013accurate}, and Tree-SHAP~\cite{lundberg2018consistent}. Additional earlier works are discussed in Appendix~\ref{apd:history}.

\vspace{-0.05in}
\section{Discussion}
\vspace{-0.05in}

Understandable and accessible explanations are  cornerstones of interpretability which informed our  isolation and disjoint designs of {\archattribute} and {\archdetect}, respectively.
  Here, we develop  an interpretable, model-agnostic, axiomatic, and efficient interaction explainer which achieves state-of-the-art results on multiple attribution tasks. In addition, we introduce a new axiom and generalize existing axioms to higher-order interaction settings. This provides guidance on how to design interaction attribution methods.
  To be able to solve the transparency issue, we need to understand feature attribution better. This work proposes interpretable and axiomatic feature interaction explanations to motivate future explorations in this area.

\section*{Broader Impact}

The purpose of this work is to provide new insights into existing and future prediction models. The explanations from {\framework} can be used by both machine learning practitioners and audiences without background expertise. The societal  risk of this work is any overdependence on {\framework}. Users of this explanation method should consider the merits of not only this method but also other explanation methods for their use cases. For example, users may want fine-grained pixel-level explanations of image classifications whereas our explanations may require superpixel segmentation. Nevertheless, we believe this work can help  reveal biases in prediction models, assist in scientific discovery, and stimulate discussions on how to debug models based on feature interactions.

\bibliography{refs}
\bibliographystyle{plain}

\newpage
\begin{appendix}
\section*{Appendix}
\section{Acronyms}

\begin{table}[h]
      \centering
        \caption{Acronym Definitions}
     \resizebox{0.75\columnwidth}{!}{%
    \begin{tabular}{lc}
    \toprule
    Acronym & Meaning
    \\
    \midrule 
    pos & positive \\ 
    neg & negative\\
    IG & Integrated Gradients~\cite{sundararajan2017axiomatic}    \\
    IH & Integrated Hessians~\cite{janizek2020explaining}    \\
    MAHE & Model-Agnostic Hierarchical Explanations~\cite{tsang2018can}  \\
    SI & Shapley Interaction Index\cite{grabisch1999axiomatic}  \\
    STI & Shapley Taylor Interaction Index~\cite{dhamdhere2019shapley}  \\
    SCD & Sampling Contextual Decomposition~\cite{jin2019towards}    \\
    SOC & Sampling Occlusion~\cite{jin2019towards} \\
    ANOVA & Analysis of Variance~\cite{fisher1925statistical} \\
    LIME & Locally Interpretable Model-Agnostic Explanations~\cite{ribeiro2016should} \\
    SHAP & Shapley Additive Explanations~\cite{lundberg2017unified}\\
    GA2M & Generalized Additive Model with Pairwise Interactions~\cite{lou2013accurate}\\
    MS COCO & Microsoft Common Objects in Context~\cite{lin2014microsoft} \\
    SST & Stanford Sentiment Treebank~\cite{socher2013recursive}\\
    BERT & Bidirectional Encoder Representations from Transformers~\cite{devlin2019bert}\\
    COVID & Coronavirus Disease\\
    \bottomrule
    \end{tabular}
    }
    \label{table:acronyms}
\end{table}

\section{Input Dimensionality Reduction}
\label{apd:reduction}

For a black-box model $f:\real^ {p'}\rightarrow \real$
which takes as input a vector  with $p'$ dimensions (e.g. an image, input embedding, etc.) and maps it to a scalar output (e.g. a class logit), we can make {\archdetect} more efficient by operating on a lower dimensional input encoding $\V{x}\in\real^p$ with $p$ dimensions. To match the dimensionality $p'$ of the input argument of $f$, we define a transformation function $\xi:\real^p \rightarrow \real^{p'}$ which takes the input encoding $\V{x}$ in the lower dimensional space $p$ and brings it back to the input space of  $f$ with dimensionality $p'$.
In other words,~\eqref{eq:pairwise} becomes 
{
\setlength{\abovedisplayskip}{-3.1pt}
\setlength{\belowdisplayskip}{-1.4pt}
\begin{dmath*}
    \omega_{i,j}(\V{x})= \left(\tfrac{1}{h_i h_j} \left(f'(\V{x}^{\star}_{\{i,j\}} +  \V{x}_{\setminus{\{i,j\}}}) - f'(\V{x}'_{\{i\}} + \V{x}^{\star}_{\{j\}} +  \V{x}_{\setminus{\{i,j\}}}) - f'(\V{x}^{\star}_{\{i\}} + \V{x}'_{\{j\}} +  \V{x}_{\setminus{\{i,j\}}}) + f'(\V{x}'_{\{i,j\}} +   \V{x}_{\setminus{\{i,j\}}} )\right)\right)^2,
\end{dmath*}
} where $f' = f \circ \xi$.
Correspondingly, {\archattribute}~\eqref{eq:att} becomes
{
\setlength{\abovedisplayskip}{8.5pt}
\setlength{\belowdisplayskip}{-7.5pt}
\begin{align*}
    \phi(\II) = f'(\V{x}^{\star}_\II+\V{x}'_{\setminus \II}) - f'(\V{x}').
\end{align*}}

Examples of input encodings are discussed for the following data types:

\begin{itemize}
\item For an image, we use a superpixel segmenter, which selects regions on the image. The selection is covered by the  vector $\V{x}\in \{0,1\}^{p}$, which encodes which image segments have been selected. Note that wherever $\V{x}$ is $0$ corresponds to a baseline feature value (e.g. zeroed image pixels).  
\item  For text, we use the natural correspondence between an input embedding  and a word token. The selection of input embedding vectors is also covered by the vector $\V{x}\in\{0,1\}^{p}$.
\item  For recommendation data, we use the same type of correspondence between an input embedding and a  feature field.
\end{itemize}

Similar notions of input encodings have also been used in~\cite{ribeiro2016should,tsang2020feature}.

\section{Completeness Axiom}
\label{apd:complete}

\complete*

\begin{proof}
Based on the definition of non-additive statistical interaction (Def.~\ref{def:interaction}), a function $f$ can be represented as a generalized additive function~\cite{tsang2017detecting, tsang2018neural, tsang2018can},
here on the domain of $\mathcal{X}$:
\begin{align}
    f(\V{x}) = \sum_{i=1}^{\eta}{q_{i}(\V{x}_{\II^u_i})} + \sum_{j=1}^{p}q'_j(x_j) + b,
    \label{eq:gami}
\end{align}

where $q_i(\V{x}_{\II^u_i})$ is a function of each interaction $\II^u_i$ on $\mathcal{X}$ $\forall i=1,\dots,\eta $ interactions,  $q'_j(x_j)$ is a function for each feature $\forall j=1,\dots,p$,  and $b$ is a bias. The $u$ in $\II^u$ stands for ``unmerged''.

The disjoint sets of $\mathcal{S} =\{\II_i\}_{i=1}^k$ are the result of merging overlapping interaction sets and main effect sets, so we can merge the subfunctions $q(\cdot)$ and $q'(\cdot)$ of~\eqref{eq:gami} whose input sets overlap to write $f(\V{x})$ as a sum of new functions $g_i(\V{x}_{\II_i})~\forall i =1,\dots, k$: 
    \begin{align}
        f(\V{x}) = \sum_{i=1}^{k} g_i(\V{x}_{\II_i}) + b. 
    \label{eq:general}
    \end{align}
    For some $\{g_i\}_{i=1}^k$
    of the form of~\eqref{eq:general}, we rewrite~\eqref{eq:att} by separating out the effect of index $i$:
    \begin{align}
        \phi(\II_i) &=  f(\V{x}^{\star}_{\II_i}+ \V{x}'_{\setminus {\II_i}}) - f(\V{x}')\quad\forall i = 1, \dots, k\nonumber\\
        &= \left( g_i(\V{x}^{\star}_{\II_i}) + \sum_{\substack{j=1 \\ j\neq i}}^{k} g_j(\V{x}'_{\II_j}) +b \right) - \left(  g_i(\V{x}'_{\II_i})   + \sum_{\substack{j=1 \\ j\neq i}}^{k} g_j(\V{x}'_{\II_j})  + b \right) \label{eq:separation}\\
        &= g_i(\V{x}^{\star}_{\II_i}) -  g_i(\V{x}'_{\II_i}).
        \label{eq:idv_diff}
    \end{align}  
    Since all $\II\in\mathcal{S}$ are disjoint,
    $g_j(\V{x}'_{\II_j})$ can be canceled in~\eqref{eq:separation} $\forall j$, leading to~\eqref{eq:idv_diff}. 
    The result at~\eqref{eq:idv_diff} can also be obtained with an alternative attribution approach, as shown in Corollary~\ref{thm:complete2}.
    
    Next, we compute the sum of attributions:
    \begin{align}
        \sum_{i=1}^{k}{\phi(\II_i)} &= \sum_{i=1}^k \left(g_i(\V{x}^{\star}_{\II_i}) -  g_i(\V{x}'_{\II_i})\right)\label{eq:sum_att}\\
        &= \sum_{i=1}^k g_i(\V{x}^{\star}_{\II_i}) - \sum_{i=1}^k g_i(\V{x}'_{\II_i})\label{eq:sum_separate}\\
        &=f(\V{x}^{\star}) - f(\V{x}')\nonumber
    \end{align}
\end{proof}

\section{Completeness of a Complementary Attribution Method}
\label{apd:test}

\begin{restatable}[Completeness of a Complement]{corollary}{complete2}\label{thm:complete2}
An attribution approach: $\phi(\II) =  f(\V{x}^{\star}) - f(\V{x}'_{\II}+ \V{x}^{\star}_{\setminus \II})$, similar to what is mentioned in ~\cite{li2016understanding,jin2019towards}, also satisfies the completeness axiom.
\end{restatable}

\begin{proof}
    Based on Eqs.~\ref{eq:general} - \ref{eq:idv_diff} of Lemma~\ref{thm:complete}:
    \begin{align*}
        \phi(\II_i) &=  f(\V{x}^{\star}) - f(\V{x}'_{\II_i}+ \V{x}^{\star}_{\setminus \II_i})\\
        &= \left( g_i(\V{x}^{\star}_{\II_i}) + \sum_{\substack{j=1 \\ j\neq i}}^{k} g_j(\V{x}^{\star}_{\II_j}) +b \right) - \left( g_i(\V{x}'_{\II_i}) + \sum_{\substack{j=1 \\ j\neq i}}^{k} g_j(\V{x}^{\star}_{\II_j}) +b \right) \\
        &= g_i(\V{x}^{\star}_{\II_i}) -  g_i(\V{x}'_{\II_i})
    \end{align*}    
    We can then resume with~\eqref{eq:sum_att} of Lemma~\ref{thm:complete}.
\end{proof}

\section{Set Attribution Axiom}
\label{apd:set_attribution}
\attribution*
\set*
\begin{proof}
From~\eqref{eq:idv_diff} in Lemma~\ref{thm:complete}, {\archattribute} can be written as 
\begin{align*}
    \phi(\II_i)
= g_i(\V{x}^{\star}_{\II_i}) -  g_i(\V{x}'_{\II_i})\quad \forall i =1, \dots, k,
\end{align*}
where $f(\V{x}) = \sum_{i=1}^{k} g_i(\V{x}_{\II_i}) + b$. Since $\mathcal{S} = \{\II_i\}_{i=1}^k$ are disjoint feature sets for the same function $f$ in Axiom~\ref{thm:attribution}, $g_i(\cdot)$ and $\varphi_i(\cdot)$ are related by a constant bias $b_i$:
\begin{align*}
    \varphi_i(\V{x}) = g_i(\V{x}) + b_i
\end{align*}
Each $\varphi_i(\cdot)$ has roots, so $g_i(\V{x}) + b_i$ has roots.  $\V{x}'$ is set such that $\varphi_i(\V{x}'_{\II_i}) = g_i(\V{x}'_{\II_i}) + b_i = 0$. Rearranging,
\begin{align*}
-g_i(\V{x}'_{\II_i})=b_i.
\end{align*}
Adding $g_i(\V{x}_{\II_i}^{\star})$ to both sides,
\begin{align*}
g_i(\V{x}_{\II_i}^{\star})-g_i(\V{x}'_{\II_i})=g_i(\V{x}_{\II_i}^{\star})+b_i,
\end{align*}
which becomes
\begin{align*}
\phi(\II_i)=\varphi_i(\V{x}_{\II_i}^{\star})\quad\forall i =1, \dots, k.
\end{align*}

\end{proof}

\subsection{Set Attribution Counterexamples}

\label{apd:counter}
We now provide counterexamples to identify situations in which the related methods do not satisfy the Set Attribution axiom.

Let 
\[
f(\V{x}) = \text{ReLU}(x_1+x_3+1) + \text{ReLU}(x_2) + 1.
\]
$f(\V{x})$ can be written as $f(\V{x}) = \varphi_1(\V{x}_{\{1,3\}}
) + \varphi_2(\V{x}_{\{2\}})$ where
$\varphi_1(\V{x}) = \text{ReLU}(x_1+ x_3+1)$, and
$\varphi_2(\V{x}) = \text{ReLU}(x_2)+1$.
According to the Set Attribution axiom, an interaction attribution method admits attributions as 
\begin{itemize}
    \item $\text{ReLU}(x_1+x_3+1)$ for features $\II_1=\{1, 3\}$ 
    \item $\text{ReLU}{(x_2)}+1$ for feature $\II_2=\{2\}$.
\end{itemize}

The above setting serves as counterexamples to the related methods as follows:

\begin{itemize}
    \item CD always assigns $\alpha + \frac{\alpha}{\alpha + \beta}$ to $\II_1$ and $\beta + \frac{\beta}{\alpha + \beta}$ to $\II_2$, where $\alpha = \text{ReLU}(x_1 + x_3+1)$ and $\beta=\text{ReLU}(x_2)$.
    \item SCD uses an expectation over an activation decomposition, which does not guarantee  admission of $\text{ReLU}(x_1+x_3+1)$ for $\II_1$ and $\text{ReLU}(x_2)$ for $\II_2$ through their respective decompositions. In the ideal case SCD becomes CD, which still does not satisfy Set Attribution from above.
    \item IH always assigns a zero attribution to $\II_2$ from hessian computations. IH also does not assign attributions to general sets of features.
    \item SOC does not assign attributions to general feature sets, only contiguous feature sequences.
    \item Both SI and STI assign the following attribution score to $\II_1$:
    \begin{align}
    \text{ReLU}(x_1+x_3+1) - \text{ReLU}(x_1+x_3'+1) - \text{ReLU}(x_1'+x_3+1) + \text{ReLU}(x_1'+x_3'+1).
    \label{eq:relu_terms}
        \end{align}
There do not exist a selection of $x_1'$ and $x_3'$  such that this attribution becomes $\text{ReLU}(x_1+x_3+1)$ for all values of $x_1$ and $x_3$.
    \begin{proof}
    We prove via case-by-case contradiction.
    Only the $\text{ReLU}(x_1+x_3+1)$ term can create  an interaction between $x_1$ and $x_3$, and this term is also the target result, so any nonzero deviation from this term via independent $x_1$ or $x_3$ effects in~\eqref{eq:relu_terms} must be countered. These independent effects manifest as the $\text{ReLU}(x_1+x_3'+1)$  or $\text{ReLU}(x_1'+x_3+1)$ terms respectively. Since ReLU is always non-negative, the only way either of these terms is nonzero is if it is positive, which implies that $\text{ReLU}(x_1+x_3'+1) = x_1+x_3'+1$ or $\text{ReLU}(x_1'+x_3+1) = x_1'+x_3+1$.  If both terms are positive, their substitution into ~\eqref{eq:relu_terms} yields $\text{ReLU}(x_1+x_3+1) - x_1-x_3'-1 - x_1' - x_3 -1 + \text{ReLU}(x_1'+x_3'+1)$. Even if $\text{ReLU}(x_1'+x_3'+1)$ is positive, we obtain $\text{ReLU}(x_1+x_3+1) - x_1-x_3'-1 - x_1' - x_3 -1 + x_1'+x_3'+1 =\text{ReLU}(x_1+x_3+1)  -x_1-x_3-1$. Asserting $-x_1-x_3-1=0$ is a contradiction.
    If only one of the  independent effects was positive, we also cannot assert $0$ through similar simplifications.
    
    Now consider the remaining case where $\text{ReLU}(x_1+x_3'+1) = \text{ReLU}(x_1'+x_3+1) = \text{ReLU}(x_1'+x_3'+1) = 0$. For any real-valued $x_1'$ or $x_3'$ , there can also be a negative real-valued $x_3$ or $x_1$ respectively. From either terms $\text{ReLU}(x_1+x_3'+1)$ or  $\text{ReLU}(x_1'+x_3+1)$, we obtain $\text{ReLU}(1) = 0$, which is a contradiction.
    \end{proof}
\end{itemize}

\section{Other Axioms}

\label{apd:other}

\subsection{Sensitivity Axiom}

\begin{restatable}[Sensitivity (a)]{lemma}{sensitivitya}\label{thm:sensitivit_a}
If $\V{x}^{\star}$ and $\V{x}'$ only differ at features indexed in $\II$ and $f(\V{x}^{\star})\neq f(\V{x}')$, then $\phi(\II)$~\eqref{eq:att} yields a nonzero attribution.
\end{restatable}

\begin{proof}
Since $\V{x}^{\star}$ and $\V{x}'$ only differ at $\mathcal{I}$, the following is true: $\V{x}^{\star}_{\setminus \II} = \V{x}'_{\setminus \II}  $. We can therefore write $\V{x}^{\star}$ as
\begin{align*}
    \V{x}^{\star} & = \V{x}^{\star}_\II  + \V{x}^{\star}_{\setminus\II}\\
    &= \V{x}^{\star}_\II  + \V{x}'_{\setminus\II}
\end{align*}

Substituting this equivalence in~\eqref{eq:att}, we have 
   \begin{align*}
        \phi(\II) &=  f(\V{x}^{\star}_{\II}+ \V{x}'_{\setminus {\II}}) - f(\V{x}')\\
        &=  f(\V{x}^{\star}) - f(\V{x}').
    \end{align*}

Since $f(\V{x}^{\star}) - f(\V{x}') \neq 0 $, we directly obtain $\phi(\II) \neq 0$.
    
\end{proof}

\begin{restatable}[Sensitivity (b)]{lemma}{sensitivityb}\label{thm:sensitivit_b}
If $f$ does not functionally depend on $\II$, then $\phi(\II)$ is always zero.
\end{restatable}

\begin{proof}
Since $f$ does not functionally depend on $\II$, 
\begin{align*}
 f(\V{x}^{\star}_\II + \V{x}'_{\setminus\II} )&=f(\V{x}'_\II + \V{x}'_{\setminus\II} ) \\
&=f(\V{x}')
\end{align*}
Therefore,
\begin{align*}
 \phi(\II) &=  f(\V{x}^{\star}_{\II}+ \V{x}'_{\setminus {\II}}) - f(\V{x}') =0.
\end{align*}
\end{proof}

\subsection{Implementation Invariance}

\begin{restatable}[Implementation Invariance]{lemma}{invariance}\label{thm:invariance}
For functionally equivalent models (with the same input-output mapping), $\phi(\cdot)$ are the same.
\end{restatable}

The definition of~\eqref{eq:att} only relies on function calls to $f$, which implies Implementation Invariance.

\subsection{Linearity}
\begin{restatable}[Linearity on $\mathcal{S}$]{lemma}{linearity}\label{thm:linearity}
If two models $f_1$, $f_2$ have the same disjoint feature sets $\mathcal{S}$ and $f=c_1 f_1 + c_2 f_2$ where $c_1,c_2$ are constants, then $\phi(\II) = c_1\phi_1(\II) + c_2 \phi_2(\II)~\forall \II\in\mathcal{S}$.
\end{restatable}

\begin{proof}

Since $f_1$ and $f_2$ have the same $\mathcal{S} =\{\II_i\}_{i=1}^k$, we can write $f_1$ and $f_2$ as follows via~\eqref{eq:general} in Lemma~\ref{thm:complete}:
\begin{align*}
        f_1(\V{x}) &= \sum_{i=1}^{k} g^{(1)}_i(\V{x}_{\II_i}) + b^{(1)}, \\
        f_2(\V{x}) &= \sum_{i=1}^{k} g^{(2)}_i(\V{x}_{\II_i}) + b^{(2)}.
\end{align*}
Since $f=c_1f_1 + c_2f_2$,
\begin{align}
        f(\V{x}) &= c_1 f_1(\V{x}) + c_2 f_2(\V{x}) \nonumber\\
        &=  \left(\sum_{i=1}^{k}c_1\times g^{(1)}_i(\V{x}_{\II_i}) + c_1\times b^{(1)}\right) +  \left(\sum_{i=1}^{k}c_2\times g^{(2)}_i(\V{x}_{\II_i}) + c_2\times b^{(2)}\right) \nonumber\\
        &=  \sum_{i=1}^{k} \left(c_1\times g^{(1)}_i(\V{x}_{\II_i}) + c_2\times g^{(2)}_i(\V{x}_{\II_i}) \right) + c_1 b^{(1)} + c_2 b^{(2)}.\label{eq:linear_comb}
\end{align}

By grouping terms as $g_i(\V{x}_{\II_i}) = c_1\times g^{(1)}_i(\V{x}_{\II_i}) + c_2\times g^{(2)}_i(\V{x}_{\II_i})$ and $b =  c_1 b^{(1)} + c_2 b^{(2)}$, we write~\eqref{eq:linear_comb} as 
\begin{align}
        f(\V{x})= \sum_{i=1}^{k} g_i(\V{x}_{\II_i}) + b. 
        \label{eq:invariance_gam}
\end{align}
From the form of~\eqref{eq:invariance_gam}, we can invoke~\eqref{eq:idv_diff}:
 $\phi(\II_i)= g_i(\V{x}^{\star}_{\II_i}) -  g_i(\V{x}'_{\II_i})$ via Lemma~\ref{thm:complete}. This equation is rewritten as
\begin{align*}
    \phi(\II_i)&= g_i(\V{x}^{\star}_{\II_i}) -  g_i(\V{x}'_{\II_i})\\
    &= \left(c_1\times g^{(1)}_i(\V{x}^{\star}_{\II_i}) + c_2\times g^{(2)}_i(\V{x}^{\star}_{\II_i})\right) - \left(c_1\times g^{(1)}_i(\V{x}'_{\II_i}) + c_2\times g^{(2)}_i(\V{x}'_{\II_i})\right)\\
    &= c_1\left(g^{(1)}_i(\V{x}^{\star}_{\II_i}) -  g^{(1)}_i(\V{x}'_{\II_i})\right) + c_2\left(g^{(2)}_i(\V{x}^{\star}_{\II_i}) -  g^{(2)}_i(\V{x}'_{\II_i})\right)\\
    &=  c_1\phi_1(\II_i) +  c_2\phi_2(\II_i).
\end{align*}
By noting that $\mathcal{S}=\{\II_i\}_{i=1}^k$, this concludes the proof.

\end{proof}

\subsection{Symmetry-Preserving}

We first define \emph{symmetric feature sets} as a generalization of ``symmetric variables'' from~\cite{sundararajan2017axiomatic}.
Feature index sets $\II_1$ and $\II_2$ are symmetric with respect to  function $f$ if swapping features in $\II_1$ with the features in  $\II_2$ does not change the function,  This implies that for symmetric $\II_1$ and $\II_2$,  their cardinalities are the same $\abs{\II_1} = \abs{\II_2}$, and they are disjoint sets in order to swap the features to any valid set index.

\begin{restatable}[Symmetry-Preserving]{lemma}{symmetry}\label{thm:symmetry}
 For $\V{x}^{\star}$ and $\V{x}'$ that each have identical feature values between symmetric feature sets with respect to $f$, the symmetric feature sets receive identical attributions $\phi(\cdot)$.
\end{restatable}

\begin{proof}

Since $\V{x}^{\star}$ and $\V{x}'$ each have identical feature values between the symmetric feature sets, 
\begin{align*}
    &\{x_i^{\star}\}_{i\in\II_1} = \{x_j^{\star}\}_{j\in\II_2},\\
    &\{x_i'\}_{i\in\II_1} = \{x_j'\}_{j\in\II_2}.
\end{align*}
Therefore, the symmetry implies the following for any $\V{x}$ in the domain of $f$.
\begin{align}
    f\left(\V{x}^{\star}_{\II_1} + \V{x}'_{\II_2} + \V{x}_{\setminus{(\II_1\cup \II_2)}}\right)
    =  f\left(\V{x}'_{\II_1} + \V{x}^{\star}_{\II_2} + \V{x}_{\setminus{(\II_1\cup \II_2)}}\right) 
    \label{eq:symmetry}
\end{align}
Setting $\V{x}=\V{x}'$, we rewrite~\eqref{eq:symmetry} as
\begin{align*}
    f&\left(\V{x}^{\star}_{\II_1} + \V{x}'_{\II_2} + \V{x}'_{\setminus{(\II_1\cup \II_2)}}\right) - f\left(\V{x}'_{\II_1} + \V{x}^{\star}_{\II_2} + \V{x}'_{\setminus{(\II_1\cup \II_2)}}\right) = 0\\
    &= f(\V{x}^{\star}_{\II_1}  + \V{x}'_{\setminus{\II_1}}) - f(\V{x}^{\star}_{\II_2}  + \V{x}'_{\setminus{\II_2}}) \\
    &=\left(f(\V{x}^{\star}_{\II_1}  + \V{x}'_{\setminus{\II_1}})- f(\V{x}')\right) - \left(f(\V{x}^{\star}_{\II_2}  + \V{x}'_{\setminus{\II_2}})- f(\V{x}')\right) \\
    &= \phi(\II_1) - \phi(\II_2)
\end{align*}

Therefore, $\phi(\II_1) = \phi(\II_2)$.

\end{proof}

\section{Discrete Mixed Partial Derivatives Detect Non-Additive Statistical Interactions}
\label{apd:mixed}

 A generalized additive model $f_g$  is given by 
 \begin{align}
     f_g(\V{x}) = \sum_{i=1}^p g_i(x_i) + b, \label{eq:gam}
 \end{align}
 where $g_i(\cdot)$ can be any function of individual features $x_i$ and $b$ is a bias.
 Since each $x_i$ of $\V{x}\in\mathcal{X}$ only takes on two values, a line can connect all valid points in each feature. Therefore,~\eqref{eq:gam} is equivalent to 
  \begin{align}
     f_\ell(\V{x}) = \sum_{i=1}^p w_ix_i + b, \label{eq:linear}
 \end{align} for weights $w_i\in\real$ and the function domain being $\mathcal{X}$.

For the case where $p=2$, the discrete mixed partial  derivative is given by~\eqref{eq:mixed} or
\begin{align*}
      \frac{\partial^2 f }{\partial x_1 \partial x_2} = \frac{1}{h_1 h_2}\left(f([x_1^{\star}, x_2^{\star} ])  -   f([x_1^{\star}, x_2' ])  -  f([x_1', x_2^{\star} ])  +     f([x_1', x_2' ])\right),  
      \label{eq:mixed2}
\end{align*}

where $h_1 = \abs{x_1^{\star} - x_1'}$ and $h_2= \abs{x_2^{\star} - x_2'}$. 
Since any three points (not on the same line) define a plane of the form~\eqref{eq:linear} ($p=2$), we can write the fourth point as having a function value with deviation $\delta$ from the plane. 
\begin{dmath}
      \frac{\partial^2 f }{\partial x_1 \partial x_2} = \frac{1}{h_1 h_2}\left(f([x_1^{\star}, x_2^{\star} ])  -   f([x_1^{\star}, x_2' ])  -  f([x_1', x_2^{\star} ])  +     f([x_1', x_2' ])\right) \\
      =\frac{1}{h_1 h_2}\left( \left(w_1x_1^{\star} + w_2x_2^{\star} +b+ \delta\right) -  \left(w_1x_1^{\star} + w_2x_2' +b \right) -  \left(w_1x_1' + w_2x_2^{\star}+b \right) +  \left(w_1x_1' + w_2x_2' +b  \right)\right)\\
      = \frac{\delta}{h_1 h_2}.
      \label{eq:mixed3}
\end{dmath}
If~\eqref{eq:mixed3} is $0$, then $\delta=0$, which implies that $f$ can be written as~\eqref{eq:linear}. $\delta\neq0$ implies the opposite, that $f$ cannot be written in linear form (by definition). Since~\eqref{eq:linear} is equivalent to~\eqref{eq:gam} in the domain of $\mathcal{X}$, this implies that $\delta \neq 0$ if and only if $f(\V{x}) \neq g_1(x_1) + g_2(x_2) + b$.

Based on Def.~\ref{def:interaction}, we can conclude that a nonzero discrete mixed partial derivative w.r.t. $x_1$ and $x_2$ in the space $\mathcal{X}$ at $p=2$ detects a non-additive statistical interaction between the two features.

For the case where $p>2$, Def.~\ref{def:interaction} states that a pairwise interaction $\{i, j\}$ exists in $f$ if and only if $f(\V{x}) \neq f_i(\V{x}_{\setminus \{i\}}) + f_j(\V{x}_{\setminus \{j\}}) $ for functions $f_i(\cdot)$ and $f_j(\cdot)$. This means that $\{i, j\}$ is declared to be an interaction if  a local  $\{i, j\}$ interaction occurs  at any $\V{x}_{\setminus{\{i,j\}}}$, $\V{x}\in\mathcal{X}$. 

Therefore, we can detect non-additive statistical interactions $\{i,j\}$ for general $p\geq2$ via
\begin{align*}
\mathbb{E}_{\V{x}}\left[\frac{\partial^2 f }{\partial x_i \partial x_j}\right]^2 > 0,
 \end{align*}
 which mirrors the definition of pairwise interaction for real-valued $\V{x}$ in~\cite{friedman2008predictive}.

 \section{Early Works on Feature Interaction Interpretation}
 \label{apd:history}
 
 We discuss early works on feature interaction interpretation and provide a timeline for this research history in Table~\ref{table:timeline}. We also discuss  mixed partial derivatives on dichotomous variables in~\ref{apd:discrete_history}.
\subsection{Origins}
The notion of a feature interaction has been studied at least since the $19$th century when John Lawes and Joseph Gilbert used factorial designs in agricultural research at the Rothamsted Experimental Station~\cite{dean2015handbook}. A factorial design is an experiment that includes observations at all combinations of categories of each factor or feature.
However, the
``advantages [of factorial design] had never been clearly
 recognised, and many research workers believed that the best course was
 the conceptually simple one of investigating one question at a time''~\cite{yates1964sir}. In the early $20$th century, Fisher et al. ($1926$)~\cite{fisher1926048} emphasized the importance of factorial designs as being the only way to obtain information about feature interactions. Near the same time, Fisher ($1921$)~\cite{fisher1921probable} also developed one of the foundations of statistical analysis called Analysis of Variance (ANOVA) including two-way ANOVA~\cite{fisher1925statistical}, which is a factorial method to detect pairwise feature interactions based  on differences among group means in a dataset. Tukey ($1949$)~\cite{tukey1949one} extended two-way ANOVA to test if two categorical features are non-additively related to the expected value of a outcome variable. This work set a precedent for later research on detecting feature interactions based on their non-additive definition. Soon after,  experimental designs were generalized  to study feature interactions, in particular the generalized randomized block design~\cite{wilk1955randomization}, which assigns test subjects to different categories (or blocks) between features in a way where cross-categories between features serve as interaction terms in  linear regression.
 
 There was a surge of interest in improving the analysis of feature interactions after the mid $20$th century.  Belsion ($1959$)~\cite{belson1959matching} and Morgan \& Sonquist ($1963$)~\cite{morgan1963problems} proposed Automatic Interaction Detection (AID) originally under a different name. AID detects interactions by subdividing data into disjoint exhaustive subsets to model an outcome  based on categorical features. Based on AID, Kass ($1980$)~\cite{kass1980exploratory} developed Chi-square Automatic Interaction Detection (CHAID), which determines how categorical features best combine in decision trees via a chi-square test. AID and CHAID were precursors to modern decision tree prediction models. Concurrently, Nelder ($1977$)~\cite{nelder1977reformulation} introduced the ``Principle of Marginality'' arguing that a feature interaction and its marginal variables should not be considered separately, for example in linear regression. Hamada \& Wu ($1992$)~\cite{hamada1992analysis} provided a contrasting view  that an interaction is only important if one or both of its marginal variables are important. Around the same time, an influential book on interpreting feature interactions was published on how to test, plot, and understand interactions of two or three continuous or categorical features~\cite{aiken1991multiple}. 
 
 \subsection{Early \nth{21} Century Works}
 At the start of the $21$st century, efforts began to focus on interpreting interactions in accurate prediction models. Ai \& Norton ($2003$)~\cite{ai2003interaction} proposed extracting interactions from logit and probit models via mixed partial derivatives. Gevrey ($2006$) \cite{gevrey2006two} followed up by proposing mixed partial derivatives to extract interactions from multilayer perceptrons with sigmoid activations when at the time, only shallow neural networks were studied. Friedman \& Popescu ($2008$)~\cite{friedman2008predictive} proposed using hybrid models to capture interactions with decision trees and univariate effects with linear regression. Sorokina et al. ($2008$)~\cite{sorokina2008detecting} proposed to use high-performance additive trees to detect feature interactions based on their non-additive definition. At the turn of the decade, we saw Bien et al.~\cite{bien2013lasso} capture interactions with different heredity conditions using a hierarchical lasso on linear regression models. Then, Hao \& Zhang ($2014$)~\cite{hao2014interaction} drew attention towards interaction screening in high dimensional data. This summarizes feature interaction research before $2015$. 
 
 \subsection{Note on Mixed Partial Derivatives on  Dichotomous Variables}
 \label{apd:discrete_history}
To our knowledge, the usage of  mixed partial derivatives for interaction detection on  dichotomous variables (features that only take two possible values)  originated  at the turn of the \nth{21} century~\cite{grabisch1999axiomatic, ai2003interaction}, but existing methods rely on single contexts~\cite{ai2003interaction} or random contexts~\cite{grabisch1999axiomatic,dhamdhere2019shapley}. Furthermore, these methods do not consider the union of overlapping pairwise interactions for disjoint higher-order interaction detection. Our choice of contexts and our disjoint interaction detection are both important to the {\framework} framework, as we discussed in \S\ref{sec:proposeddetection} and showed through axiomatic analysis (\S\ref{sec:axioms}) and experiments (\S\ref{sec:exp_detector}).

{
\renewcommand\arraystretch{1.6}\arrayrulecolor{LightSteelBlue3}

\begin{table}[ht]
\centering
\captionsetup{singlelinecheck=false, labelfont=sc, labelsep=quad}
\center{
\caption{Timeline of research on feature interaction interpretation (Pre-$2015)$\label{table:timeline}}
}\vskip 0ex
\begin{tabular}{p{6cm} @{\hskip 9pt}  @{\,}r <{\hskip 2pt} !{\foo} >{\raggedright\arraybackslash}p{6cm}}
\toprule
\addlinespace[1.5ex]
\emph{Lawes \& Gilbert} - factorial design in agricultural research at the Rothamsted Experimental Station &1843 & \\
\emph{Fisher} -  two-way Analysis of Variance (ANOVA) & 1925 & \\
&1949 & \emph{Tukey} - Tukey’s test of additivity\\
&1955 & \emph{Wilk} - generalized random block design\\
\emph{Belson} - Automatic Interaction Detection by subdividing data &1959 & \\
\emph{Nelder} - Principle of Marginality&1977 & \\
&1980 & \emph{Kass} - Chi-square Automatic Interaction Detection by combining features in decision trees via chi-square tests\\

&1991 & \emph{Aiken \& West} - book on interpreting interaction effects\\
\emph{Hamada \& Wu} - heredity conditions &1992 & \\
 \emph{Ai \& Norton} -  interactions in logit and probit models &2003 &\\
  &2006 & \emph{Gevry et al.} - interactions in sigmoid neural networks\\
\emph{Friedman \& Popescu} - RuleFit to detect interactions by mixing linear regression and trees &2008 & \emph{Sorokina et al.} - Additive Groves to detect non-additive interactions \\
 \emph{Bien et al.} - Hierarchical Lasso&2013 & \\
\emph{Hao \& Zhang} - interaction screening in high dimensional data &2014 & \\
\end{tabular}
\end{table}
}

\renewcommand\arraystretch{1}\arrayrulecolor{black}

\newpage
\section{Attributions Compared to Annotation Labels}
\label{apd:quantile}

\begin{figure*}[ht]
    \centering
    \begin{subfigure}[t]{0.42\textwidth}
    	\vskip 0pt
        \centering
           \includegraphics[scale=0.4]{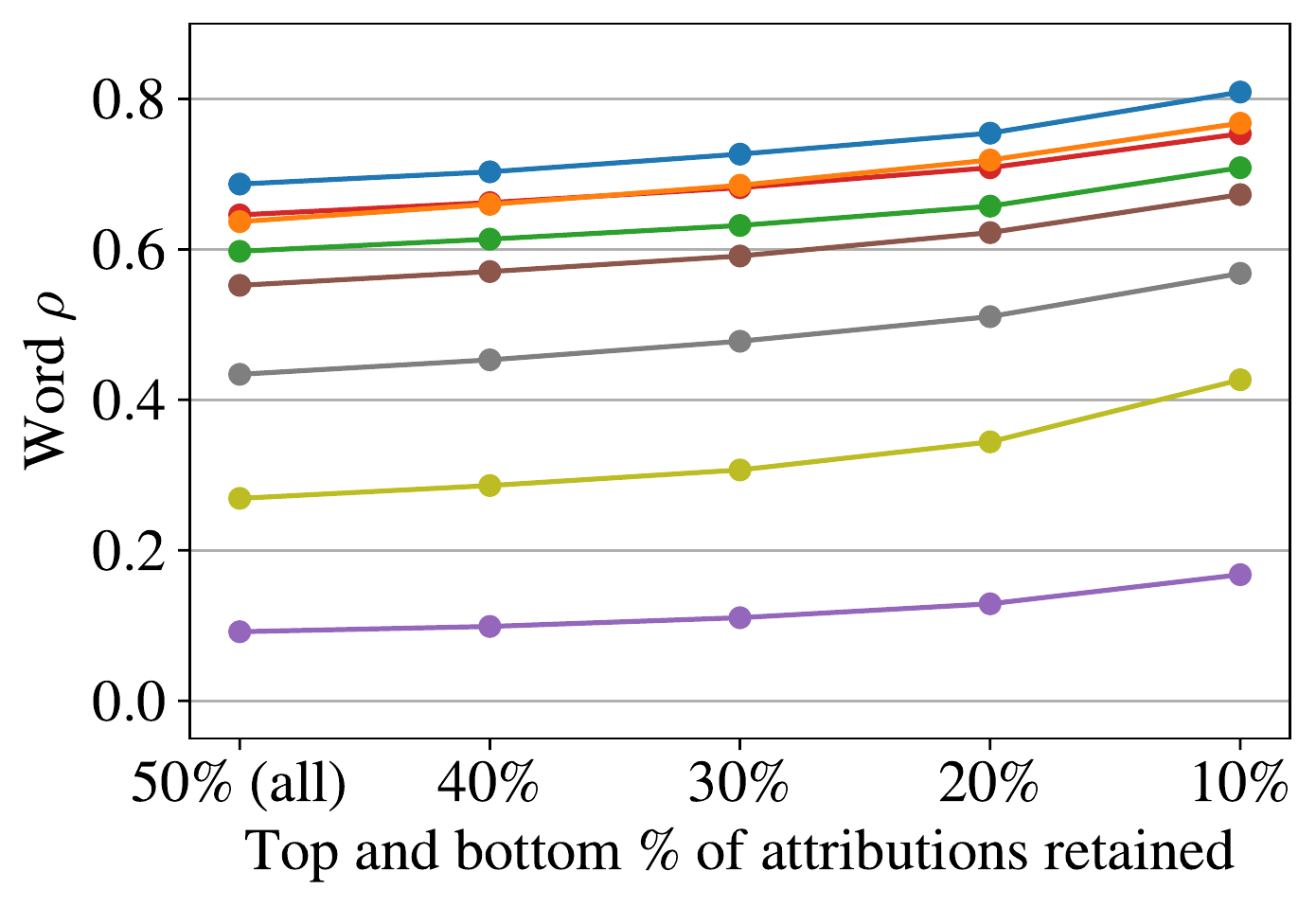}
        \caption{Word $\rho$\label{fig:sweep_bow}}
    \end{subfigure}%
    \begin{subfigure}[t]{0.58\textwidth}
     	\vskip 0pt
        \centering
        \includegraphics[scale=0.4]{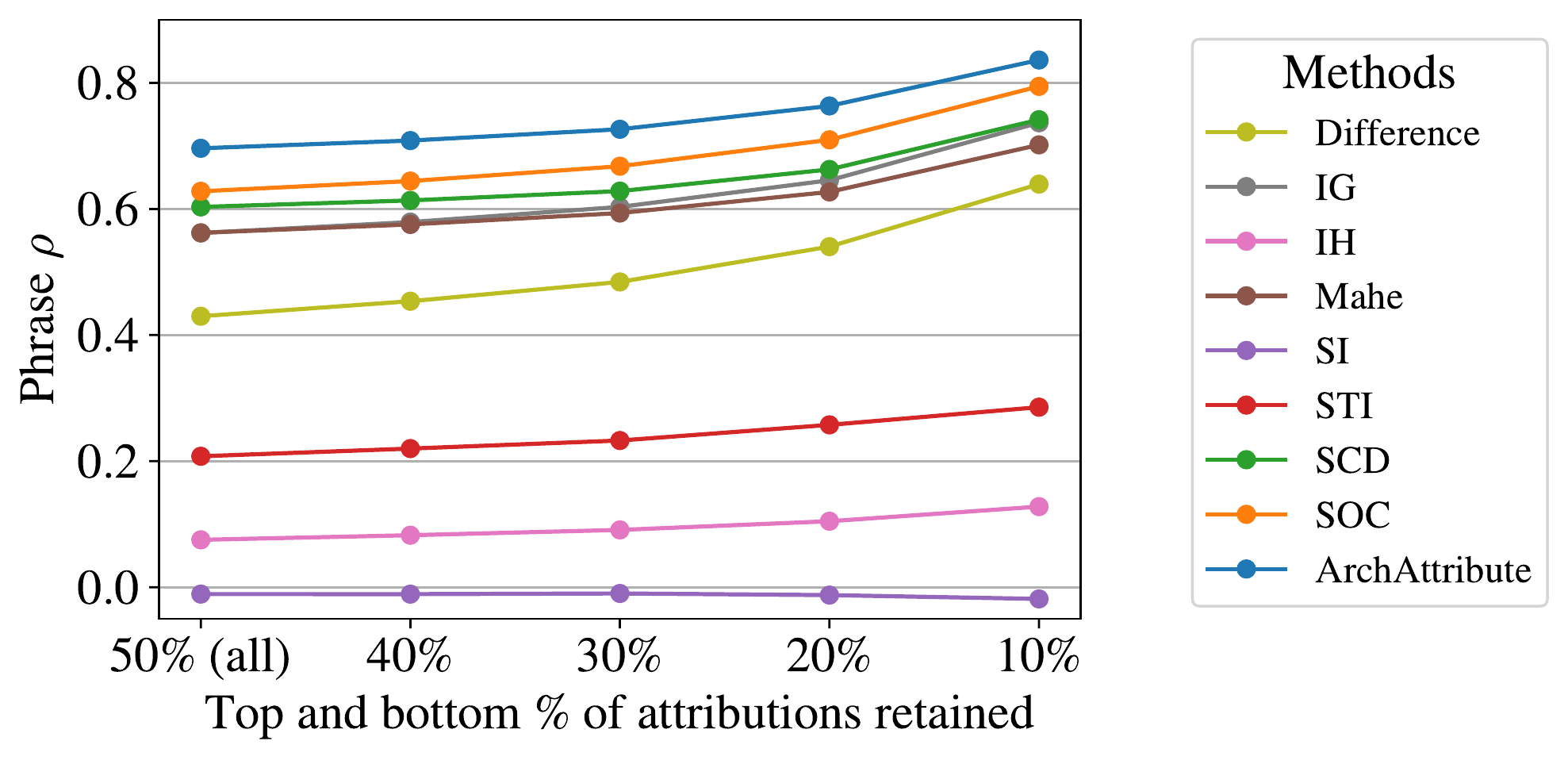}
        \caption{Phrase $\rho$\label{fig:sweep_sst}}
    \end{subfigure}
    \caption{Text explanation metrics ((a) Word $\rho$ and (b) Phrase $\rho$) versus top and bottom $\%$ of attributions retained for different attribution methods on {\bert} over the {\sst} test set. These plots expand the analysis of Table~\ref{table:eval}.
    \label{fig:text_sweep}
    }
\end{figure*}

\begin{figure}[ht]
    \centering
    \includegraphics[scale=0.4]{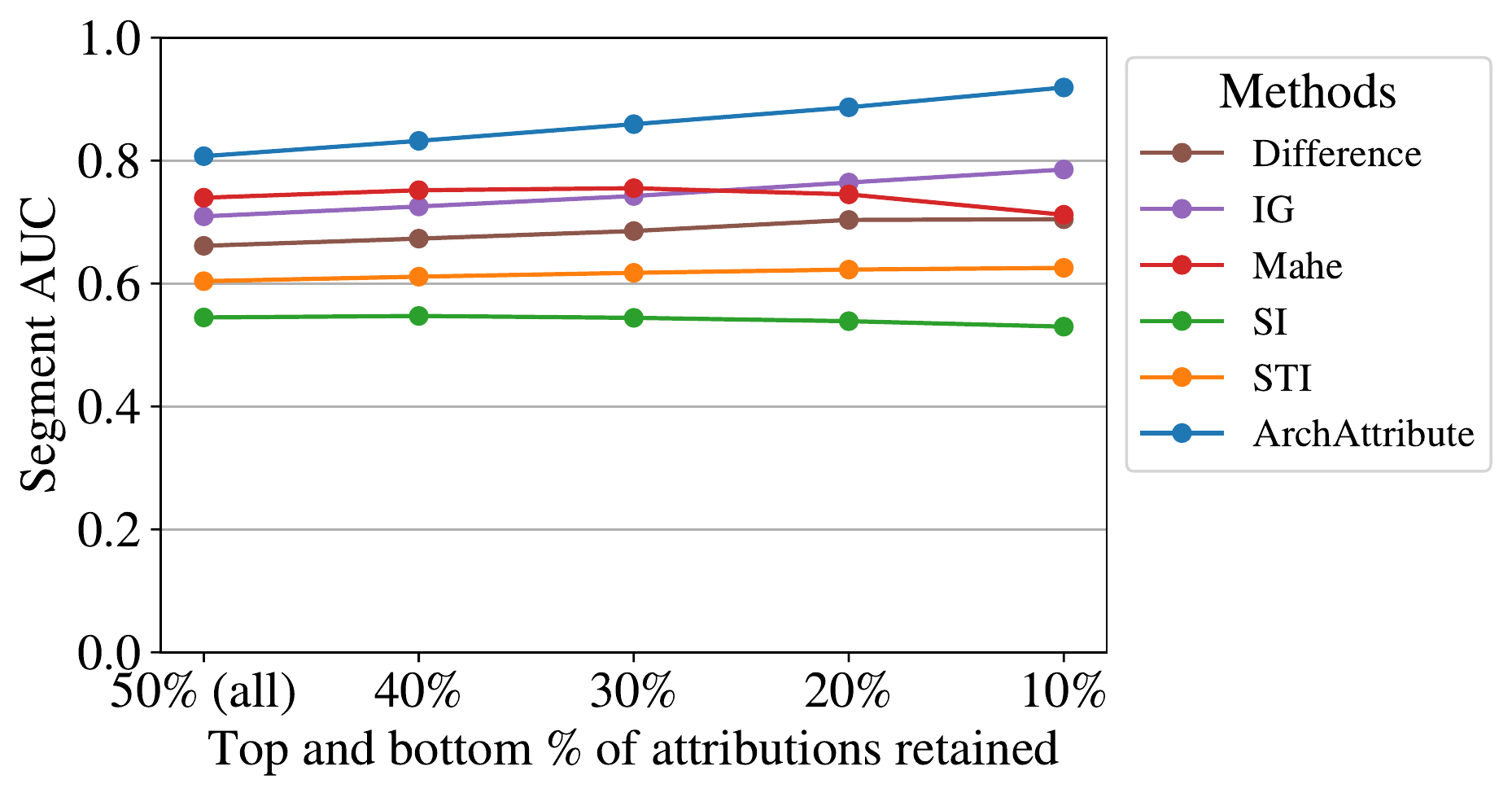}

    \caption{Image explanation metric (segment AUC) versus top and bottom $\%$ of attributions retained for different attribution methods on {\resnet} over the MS COCO test set. These plots expand the analysis of Table~\ref{table:eval}.
    \label{fig:image_sweep}
    }
\end{figure}

\section{Runtime}

\label{apd:runtime}
\begin{figure}[H]
    \centering
    \begin{subfigure}[t]{0.5\textwidth}
    	\vskip 0pt
        \centering
           \includegraphics[scale=0.36]{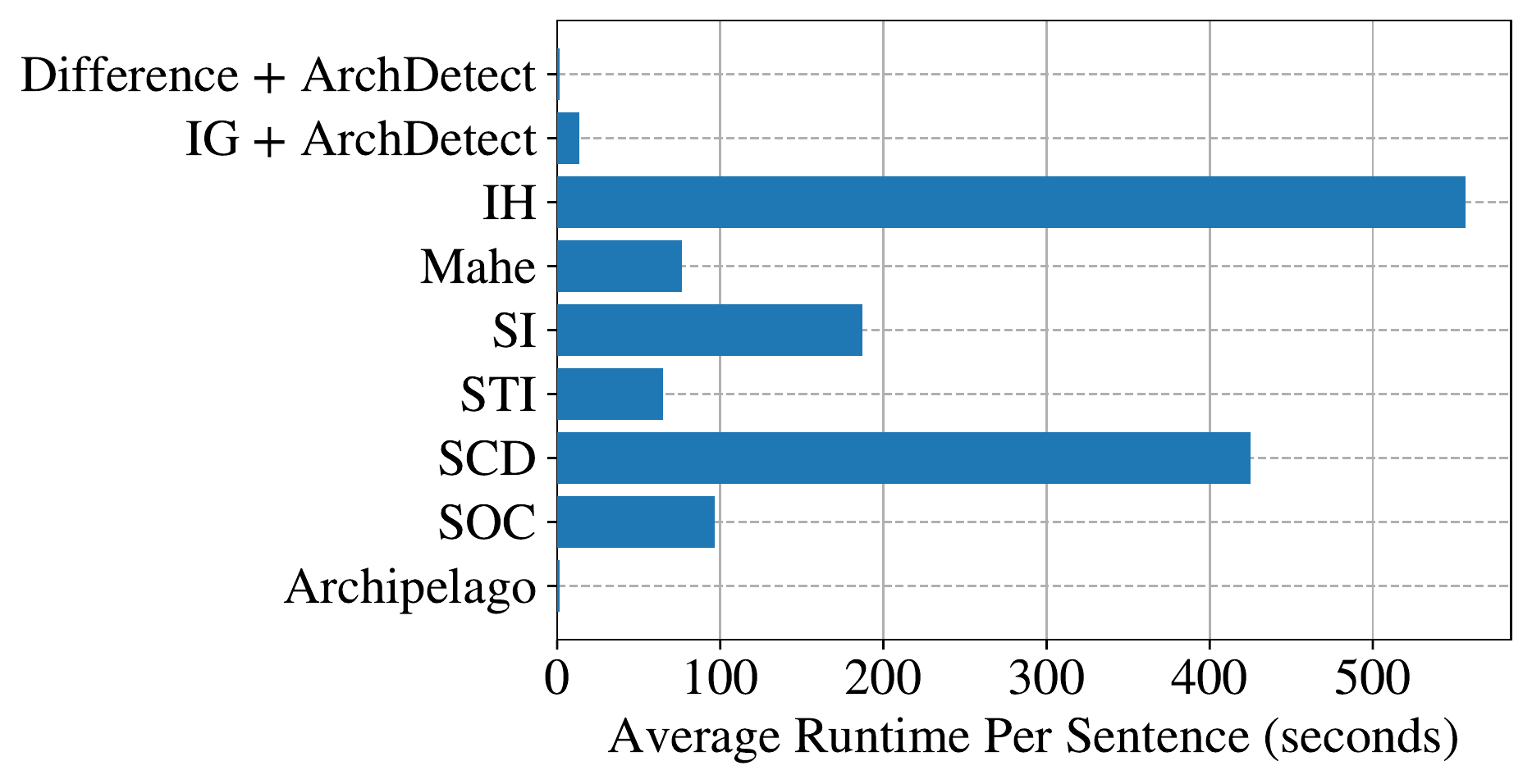}
        \caption{Sentiment Analysis on {\sst}\label{fig:runtime_text}}
    \end{subfigure}%
    \begin{subfigure}[t]{0.5\textwidth}
     	\vskip 0pt
        \centering
        \includegraphics[scale=0.36]{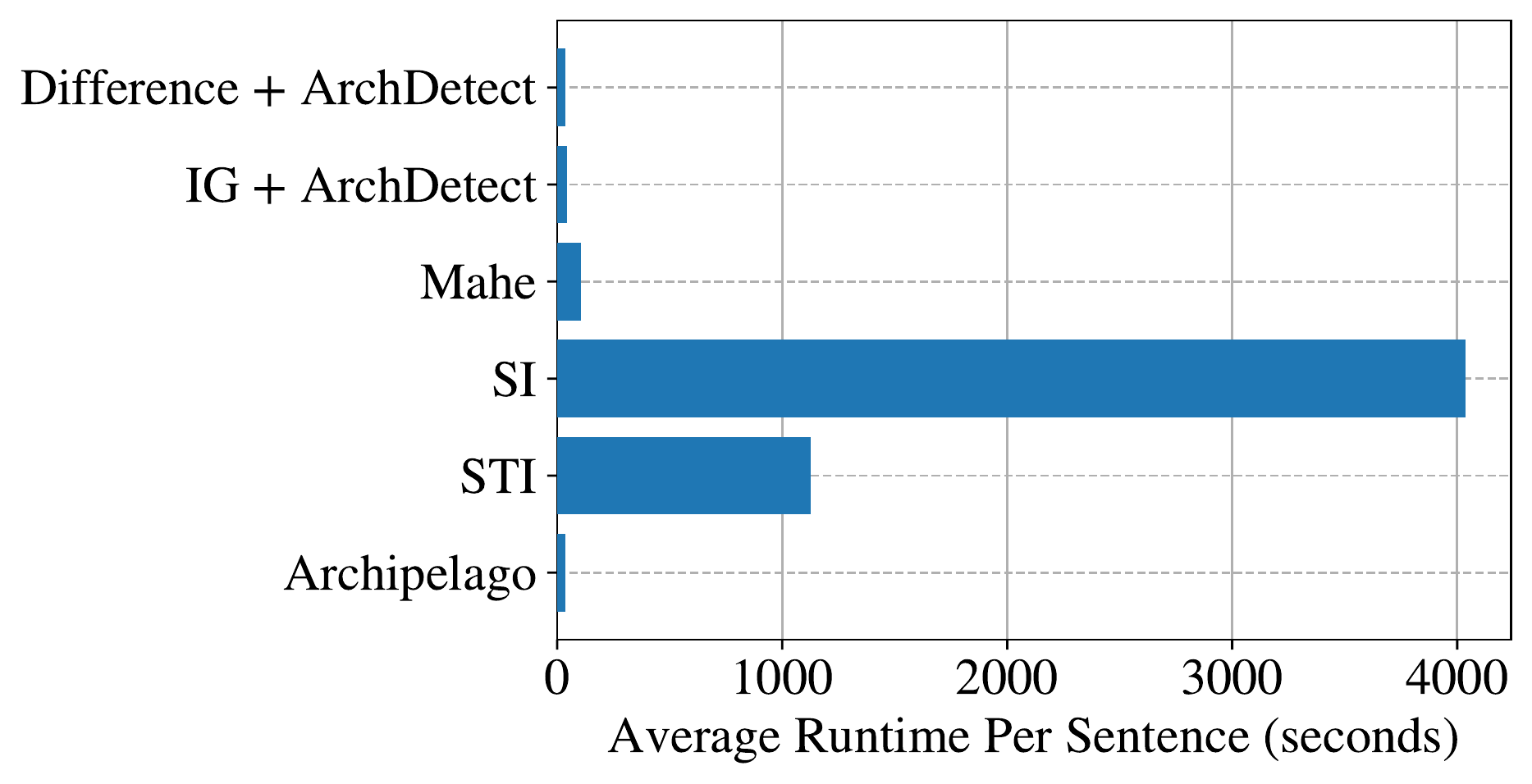}
        \caption{Image Classification on {\imagenet}\label{fig:runtime_image}}
    \end{subfigure}
    \caption{Serial runtime comparison of relevant explainer methods for (a) {\bert} sentiment analysis on {\sst} and (b) {\resnet} image classification on {\imagenet}. Runtimes are averaged across $100$ random data samples from respective test sets. These experiments were done on a server with $32$ Intel Xeon E$5$-$2640$ v2 CPUs @ $2.00$GHz and $2$ Nvidia $1080$ Ti GPUs.
    \label{fig:test}
    }
\end{figure}

\section{Visualization Comparisons}
\label{apd:viz}

\subsection{Sentiment Analysis}

Visualization comparisons  of different attribution methods on {\bert} are shown in Figs.~\ref{fig:text_viz_a}-\ref{fig:text_viz_e} for random test sentences from {\sst}. The visualization format is the same as Fig.~\ref{fig:text_exp}. Note that all \emph{individual} feature attributions that correspond to stop words (from~\cite{manning2008information}) are omitted in these comparisons and Figs.~\ref{fig:motiv},~\ref{fig:text_exp}.

\subsection{Image Classification}

\begin{figure}[ht]
    \centering
    \includegraphics[scale=0.175]{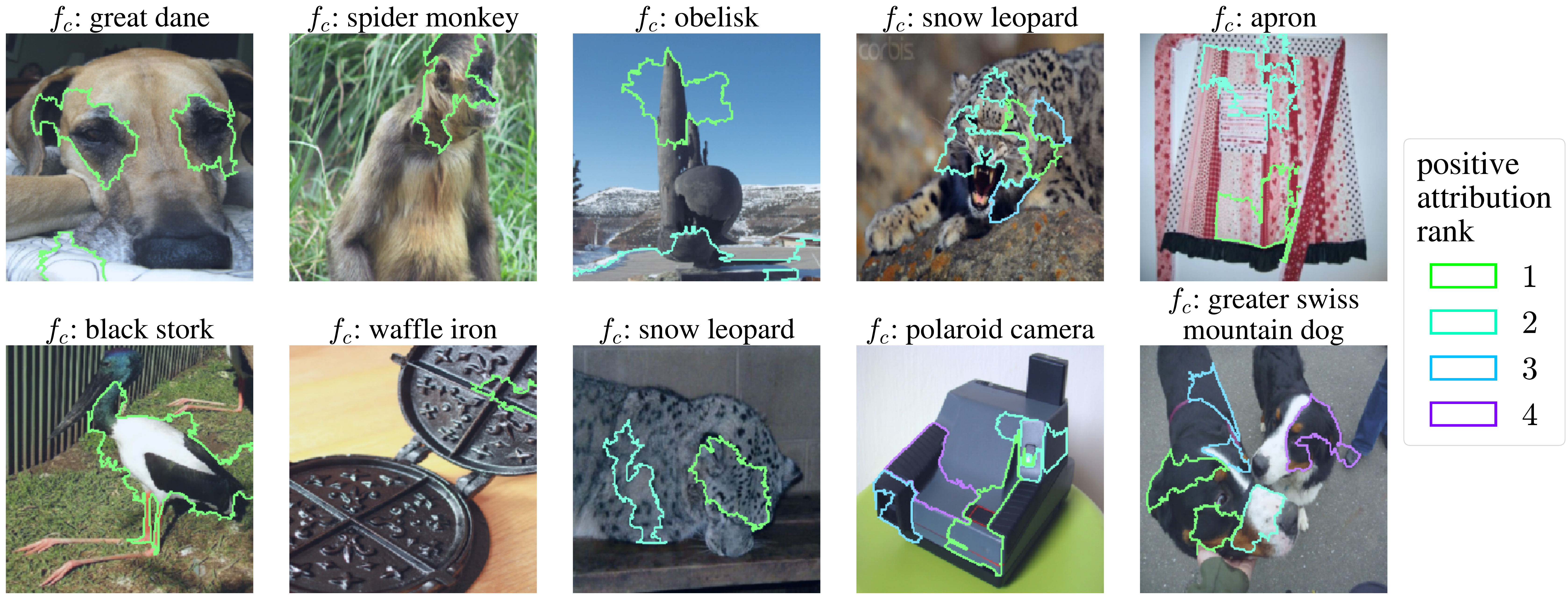}

    \caption{Our {\resnet} visualizations on random test images from {\imagenet}. Colored outlines indicate interactions with positive attribution. $f_c$ is the image classification result. To our knowledge, only this work shows interactions that support the image classification via interaction attribution.}
    \label{fig:apd_img_exp}
\end{figure}    

In Fig.~\ref{fig:apd_img_exp}, we  visualize {\framework} explanations on  $\mathcal{S}$ via top-$5$ pairwise interactions (\S\ref{sec:disjoint}), where  positive attribution interactions are shown for clarity. The images are randomly selected from the {\imagenet} test set. It is interesting to see which image parts interact, such as the eyes of the ``great dane'' image.

Visualization comparisons  of different attribution methods on {\resnet} are shown in Figs.~\ref{fig:image_viz_a}-\ref{fig:image_viz_e} for the same random test images from {\imagenet}.

\begin{figure}[ht]
    \includegraphics[scale=0.214,trim={0 0 20cm 4cm},clip]{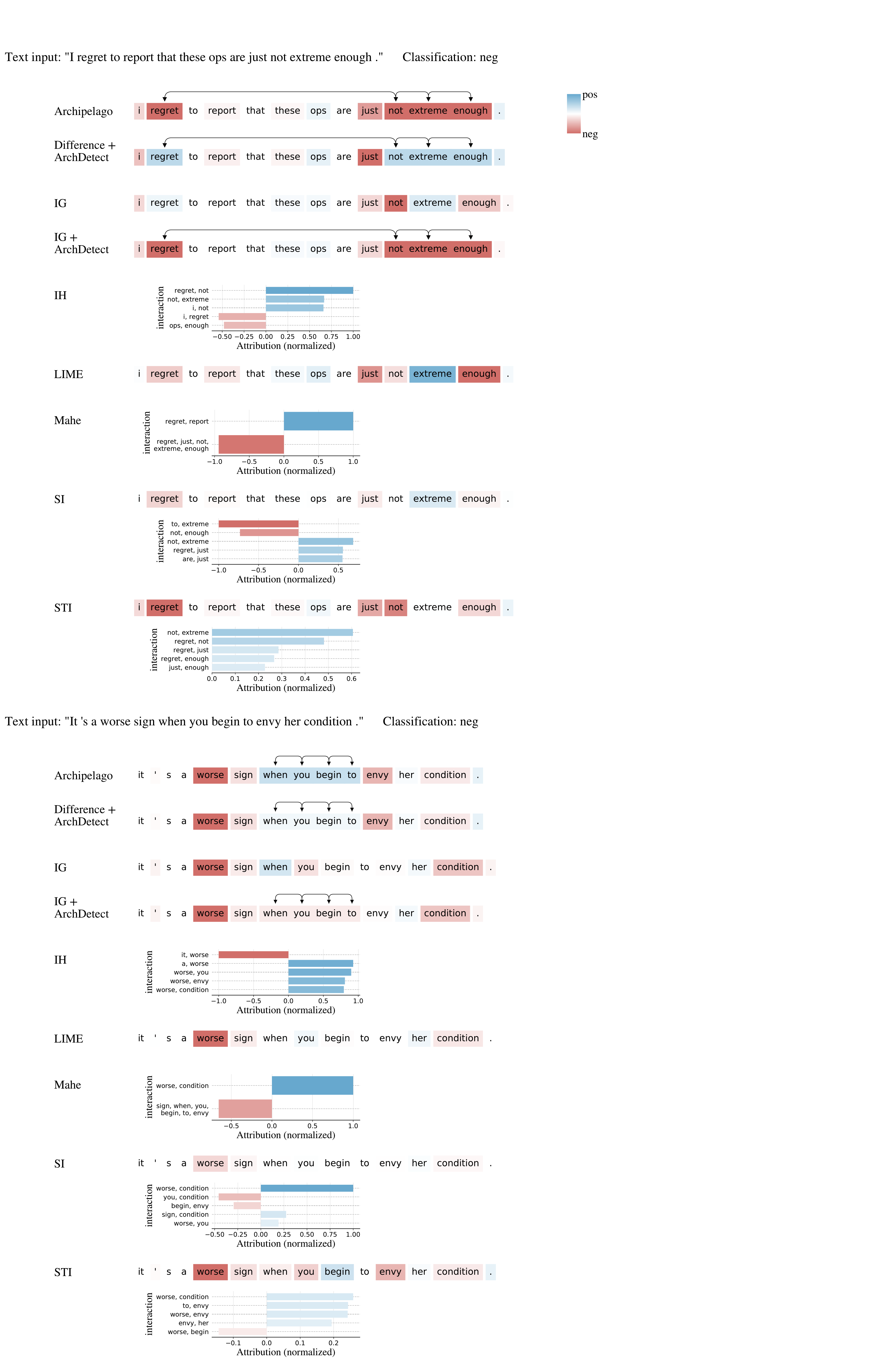}
    
    \caption{Text Viz. Comparison A. In the first text example, ``regret, not extreme enough'' is a meaningful and strongly negative interaction. In the second example, ``when you begin to'' interacts to diminish its overall attribution magnitude.
    \label{fig:text_viz_a}
    }
    \vspace{-0.1in}
\end{figure}

\begin{figure}[ht]
    \includegraphics[scale=0.214,trim={0 0 16cm 4cm},clip]{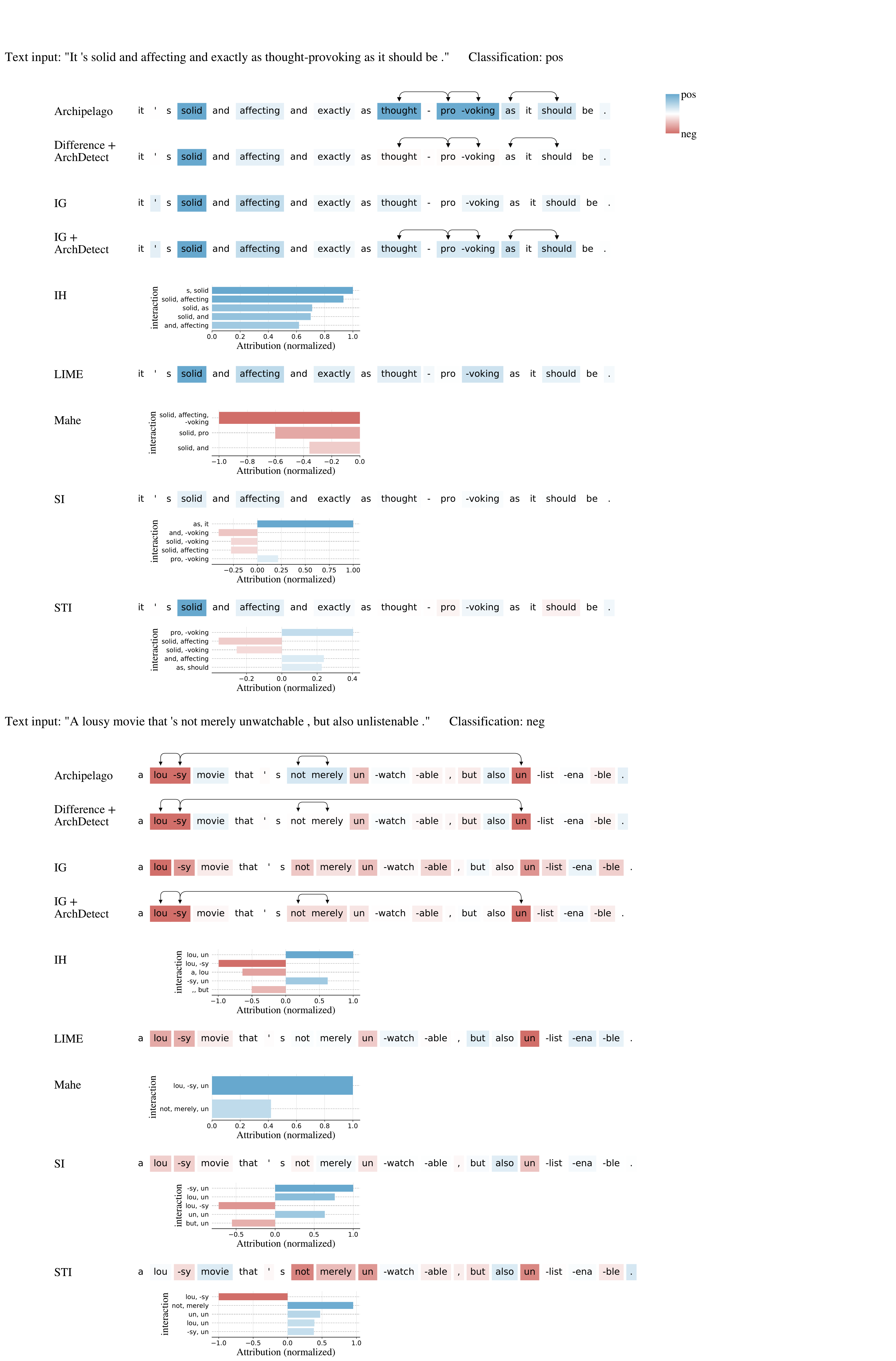}
    \caption{Text Viz. Comparison B. In the first text example, ``thought provoking'' is a meaningful and strongly positive interaction. In the second example, the ``lousy, un'' interaction factors in a large context to make a negative text classification.
    \label{fig:text_viz_b}
    }
    \vspace{-0.1in}
\end{figure}

\begin{figure}[ht]
    \includegraphics[scale=0.214,trim={0 0 12cm 4cm},clip]{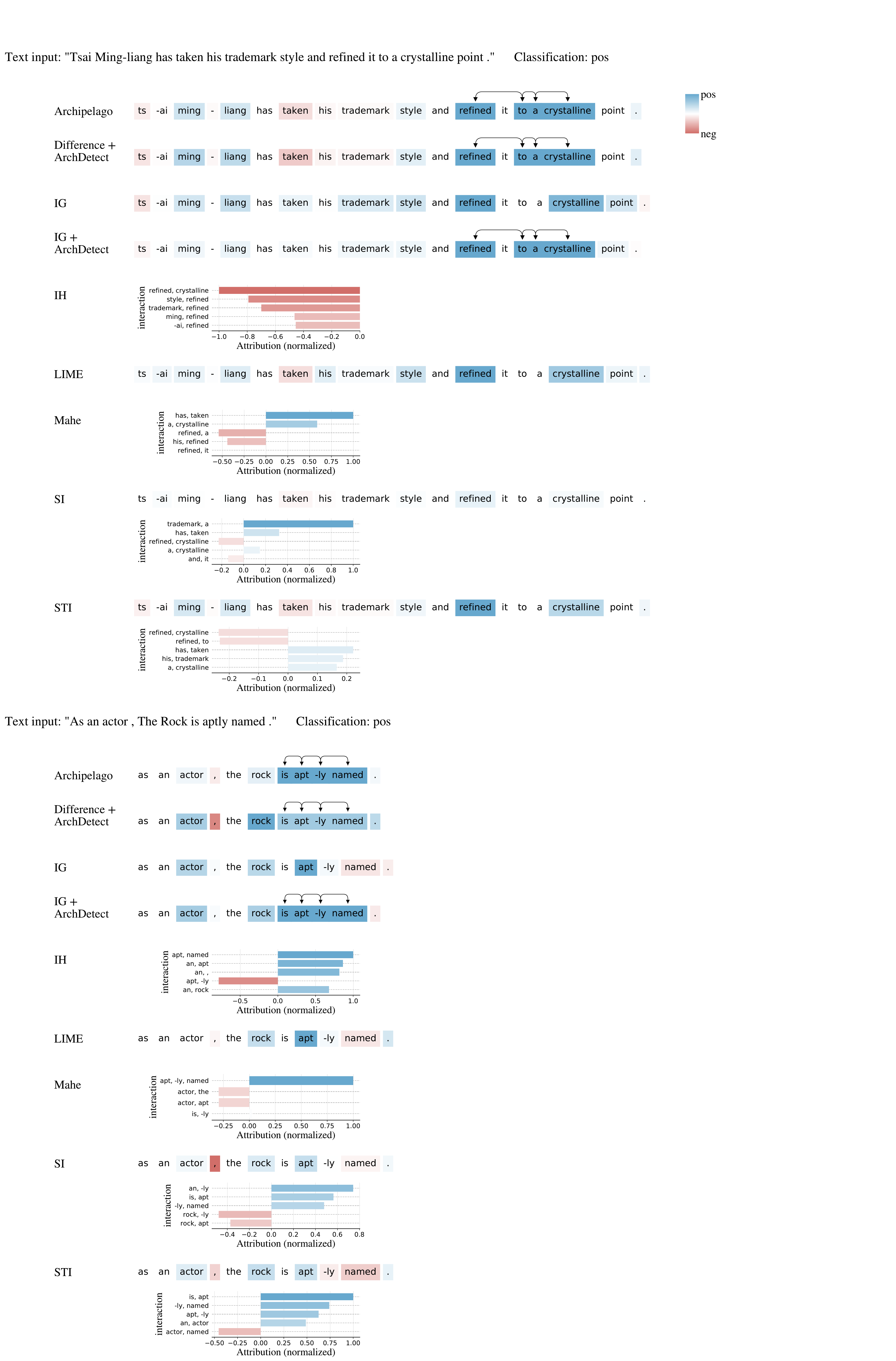}
    \caption{Text Viz. Comparison C. In the first text example, ``refined, to a crystalline'' is a meaningful and strongly positive interaction. In the second example, ``is aptly named'' is also a meaningful and strongly positive interaction.
    \label{fig:text_viz_c}
    }
    \vspace{-0.1in}
\end{figure}

\begin{figure}[ht]
    \includegraphics[scale=0.214,trim={0 0 14cm 4cm},clip]{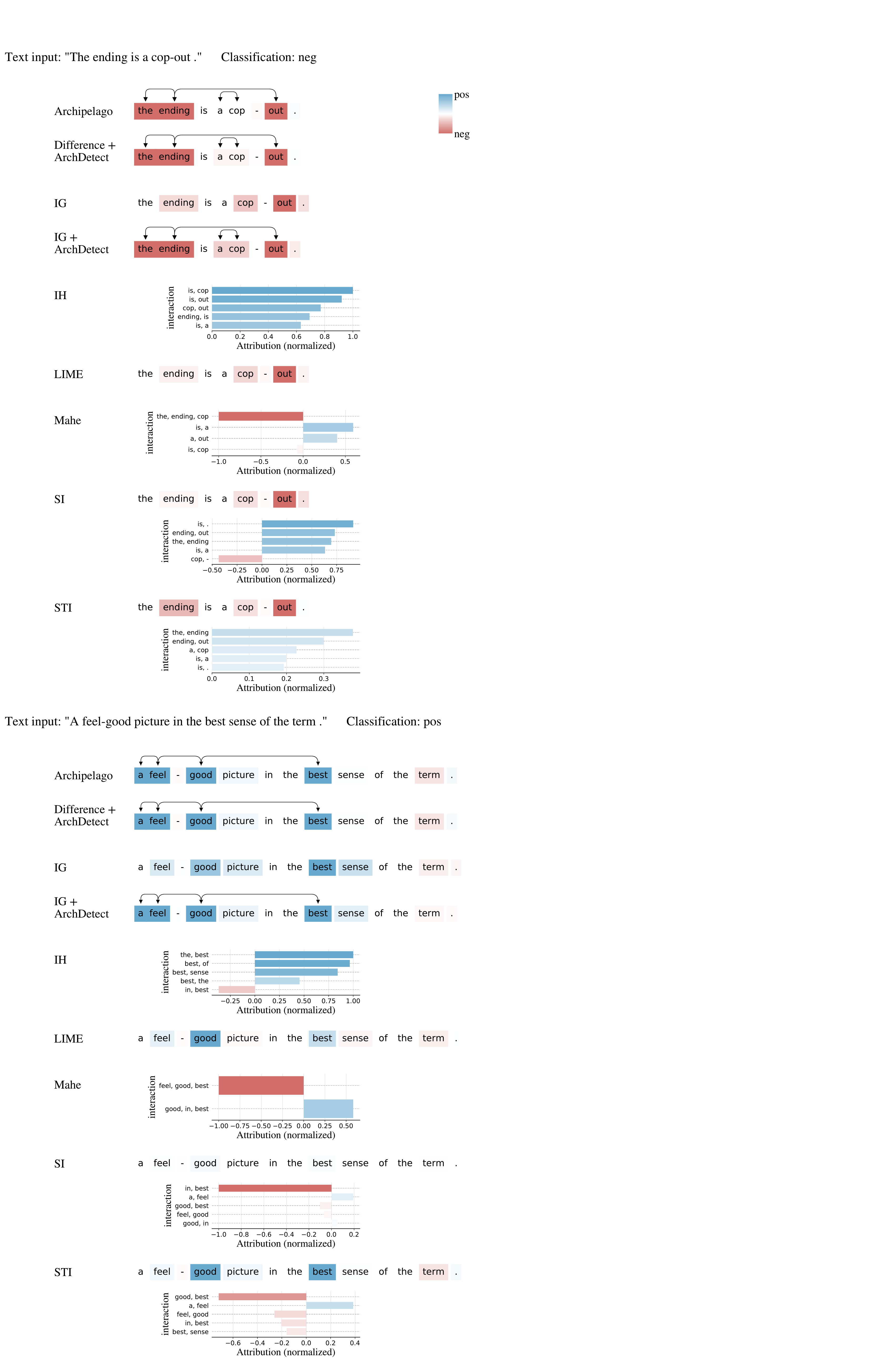}
    
    \caption{Text Viz. Comparison D. In the first text example, ``the ending, out'' is a meaningful and negative interaction. In the second example, ``a feel good, best'' is a meaningful and strongly positive interaction.
    \label{fig:text_viz_d}
    }
    \vspace{-0.1in}
\end{figure}

\begin{figure}[ht]
    \includegraphics[scale=0.214,trim={0 0 14cm 4cm},clip]{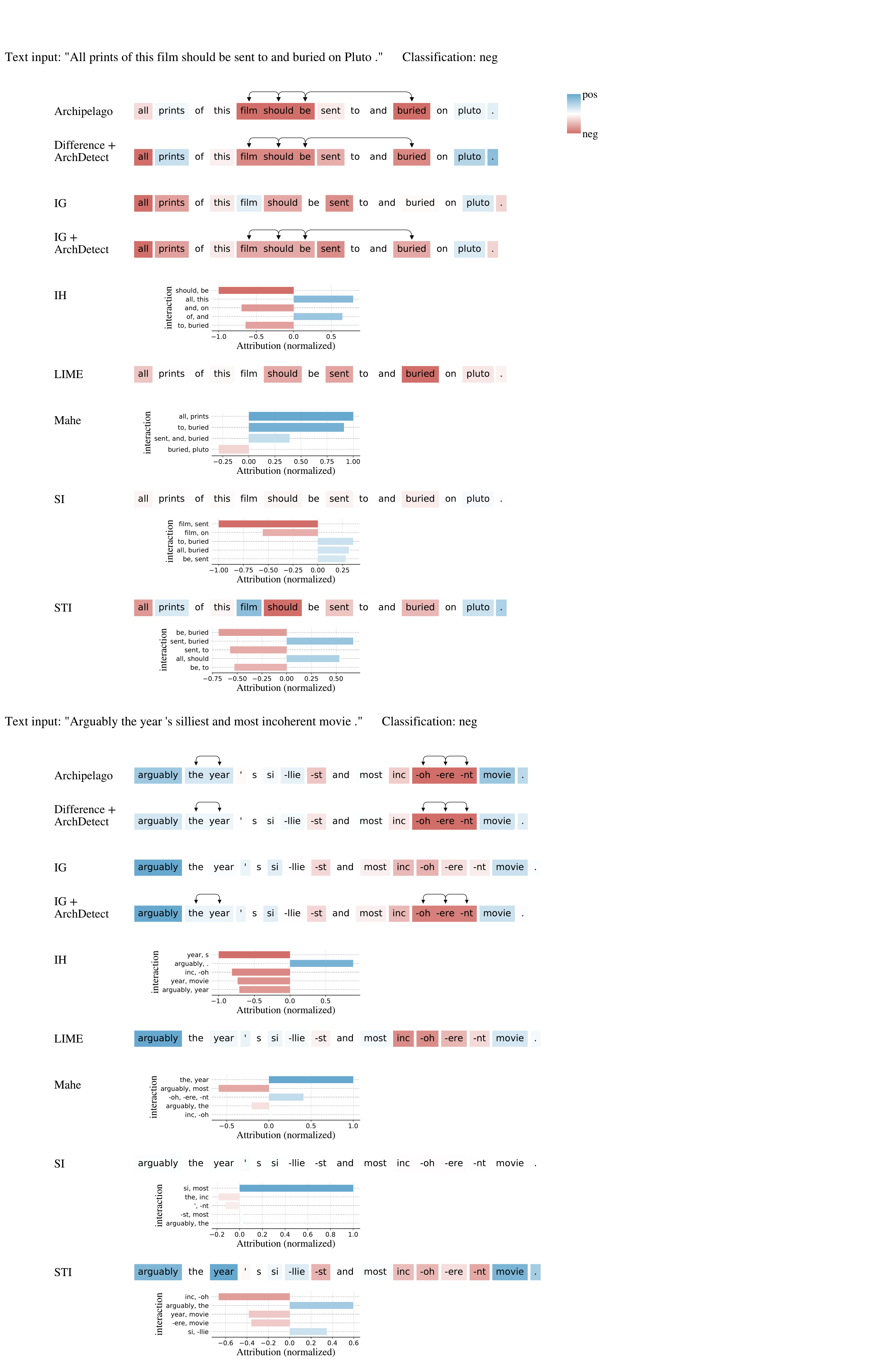}
    
    \caption{Text Viz. Comparison E. In the first text example, ``film should be, buried'' is a meaningful and strongly negative interaction. In the second example, ``-oherent'' belongs to a negative word ``incohorent''.
    \label{fig:text_viz_e}
    }
    \vspace{-0.1in}
\end{figure}

\begin{figure}[ht]
    \includegraphics[scale=0.145,trim={0 0 0cm 0},clip]{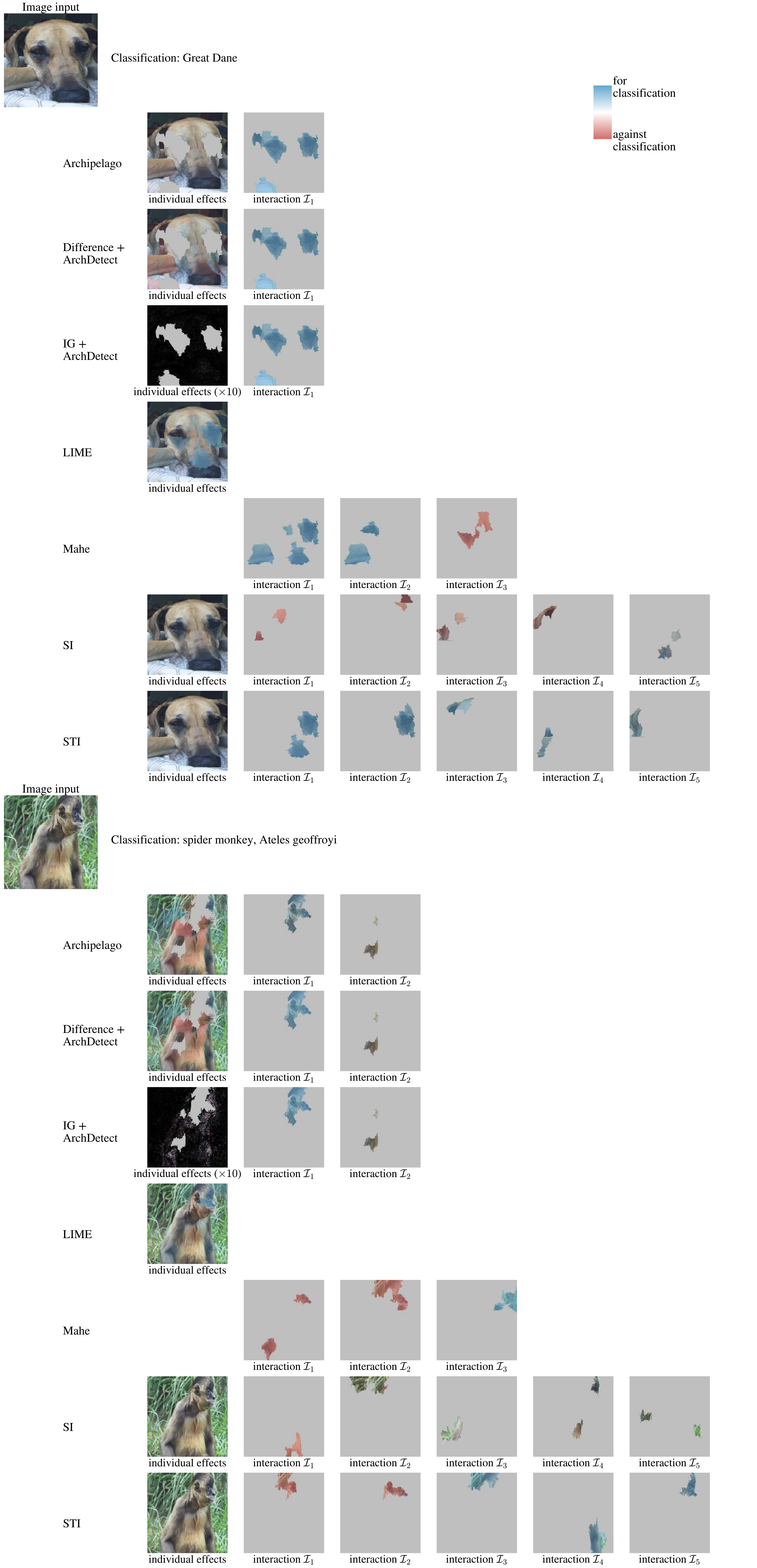}
    
    \caption{Image Viz. Comparison A. In the first image example, the dog's eyes are a meaningful  interaction supporting the classification. In the second example, the monkey's head is also a positive interaction.
    \label{fig:image_viz_a}
    }
    \vspace{-0.1in}
\end{figure}

\begin{figure}[ht]
    \includegraphics[scale=0.145,trim={0 0 0cm 0},clip]{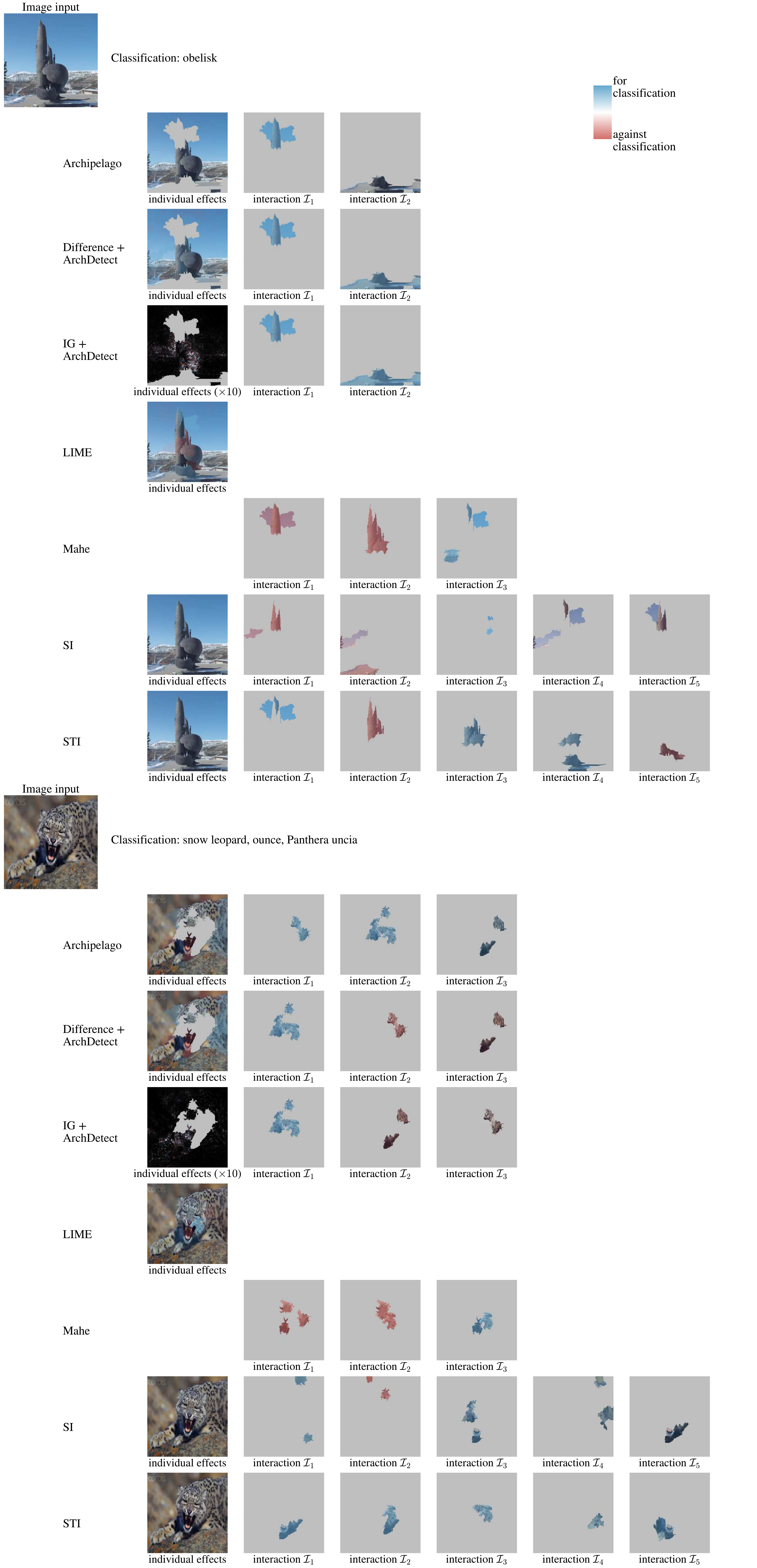}
    
    \caption{Image Viz. Comparison B. In the first image example, the obelisk tip is a meaningful  interaction supporting the classification. In the second example, the leopard's face is also a positive interaction.    
    \label{fig:image_viz_b}
    }
    \vspace{-0.1in}
\end{figure}

\begin{figure}[ht]
    \includegraphics[scale=0.145,trim={0 0 0cm 0},clip]{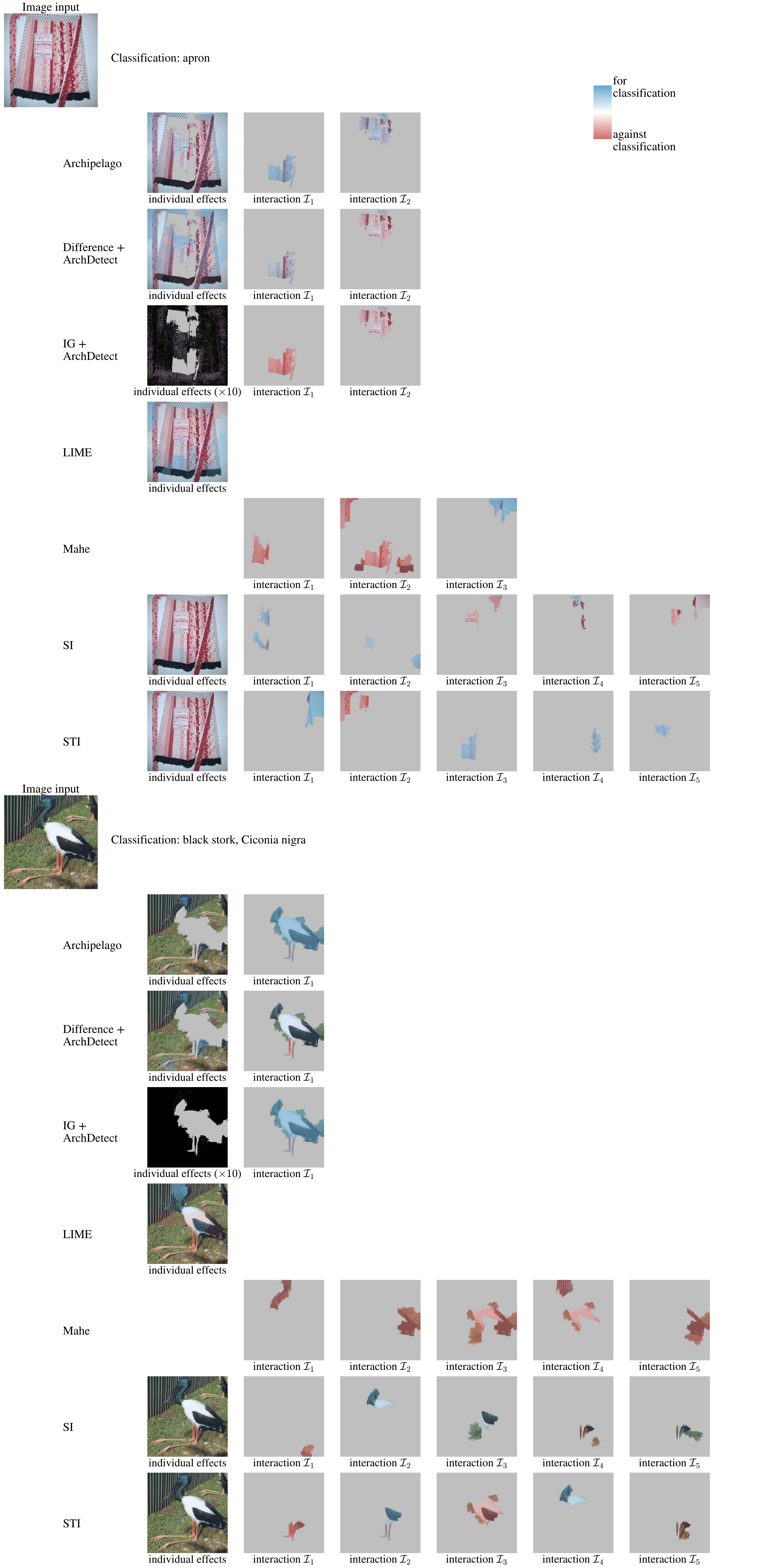}
    
    \caption{Image Viz. Comparison C. In the first image example, different patches of the apron are  interactions supporting the classification. In the second example, the stork's body is an interaction that strongly supports the classification.  
    \label{fig:image_viz_c}
    }
    \vspace{-0.1in}
\end{figure}

\begin{figure}[ht]
    \includegraphics[scale=0.145,trim={0 0 0cm 0},clip]{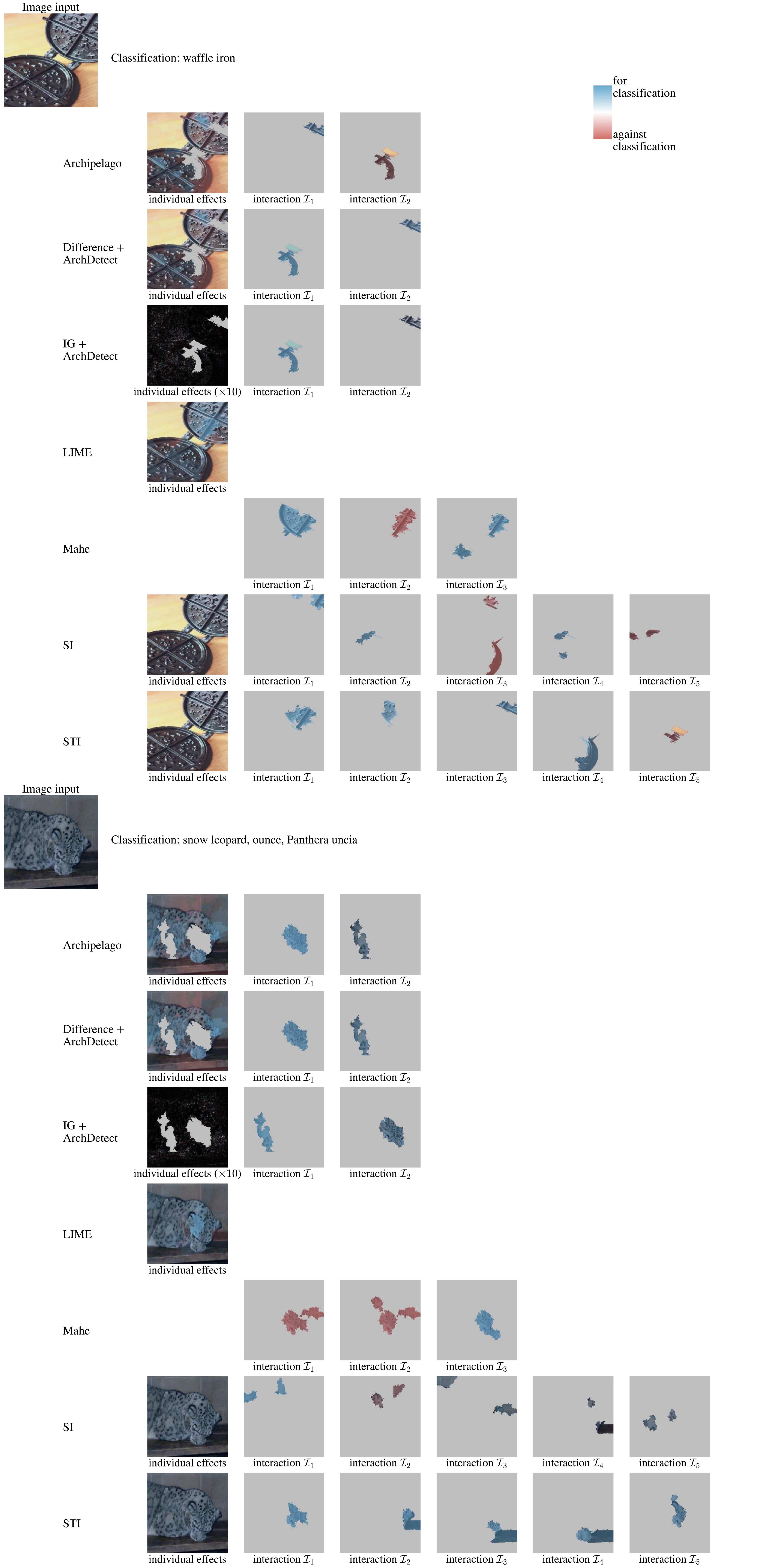}
    
    \caption{
    Image Viz. Comparison D. In the first image example, certain small patches of the waffle iron  interact, one of which supports the classification. In the second example, the leopard's face is the primary positive interaction.
    \label{fig:image_viz_d}
    }
    \vspace{-0.1in}
\end{figure}

\begin{figure}[ht]
    \includegraphics[scale=0.145,trim={0 0 0cm 0},clip]{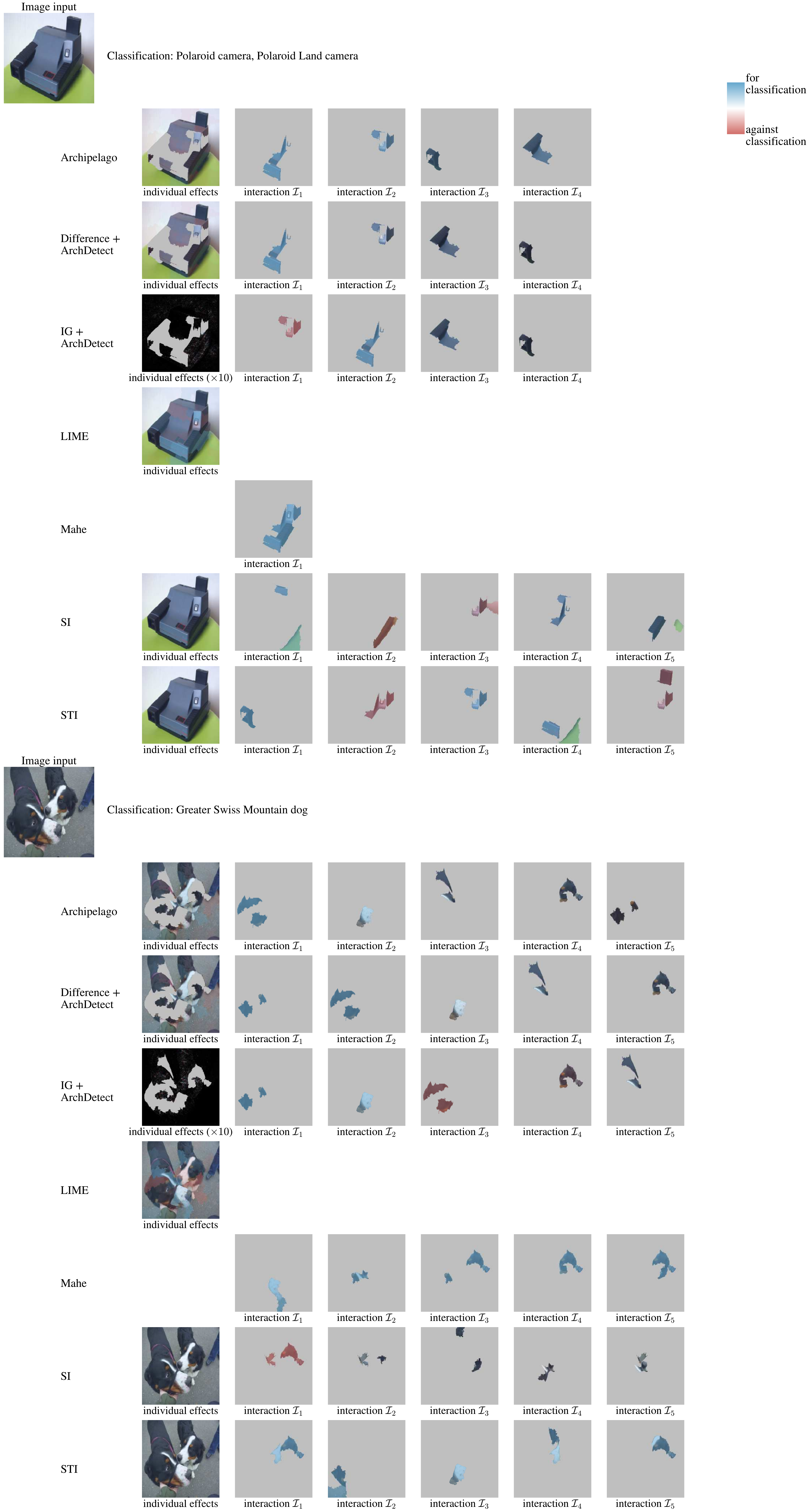}
    
    \caption{
   Image Viz. Comparison E. In the first image example, different parts of the polaroid camera are interactions that positively support the classification. In the second example, the dogs' heads and body are also positive interactions.
    \label{fig:image_viz_e}
   }
    \vspace{-0.1in}
\end{figure}

\section{{\archdetect} Ablation Visualizations}
\label{apd:ablation}

We run an ablation study removing the $\V{x}'_{\setminus\{i,j\}}$ baseline context from~\eqref{eq:efficient} for disjoint interaction detection and examine its effect on visualizations. The visualizations are shown in Fig.~\ref{fig:text_wo_viz_a} for sentiment analysis and Figs.~\ref{fig:image_wo_viz_a} and~\ref{fig:image_wo_viz_b} for image classification. Top-$3$ and top-$5$ pairwise interactions are used in sentiment analysis and image classification respectively before merging the interactions.

\begin{figure}[ht]
    \includegraphics[scale=0.22,trim={0 12cm 16cm 4cm},clip]{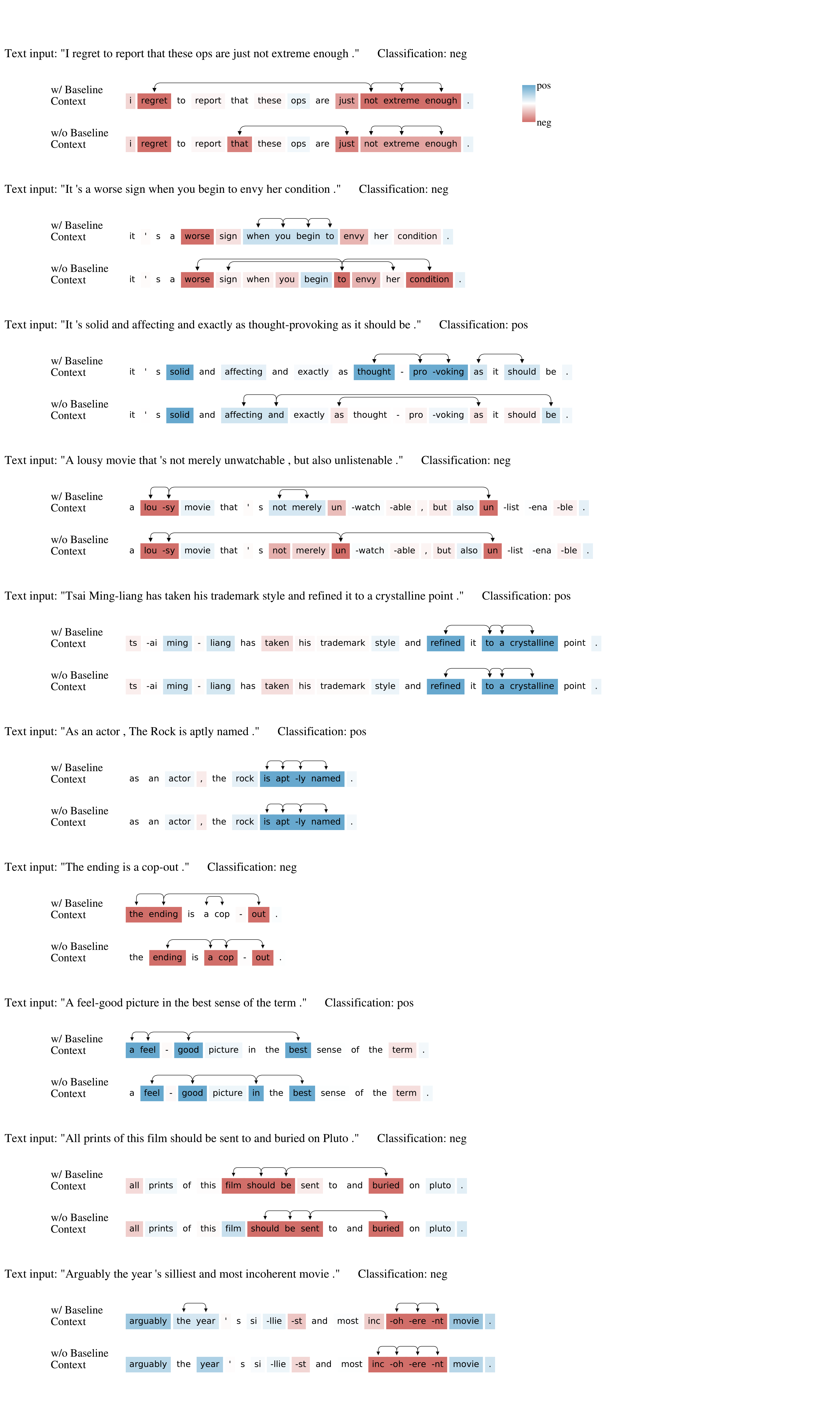}
    
    \caption{Text Viz. with {\archdetect} Ablation. The interactions tend to use more salient words when including the baseline context, which is proposed in {\archdetect}.
    \label{fig:text_wo_viz_a}
    }
    \vspace{-0.1in}
\end{figure}

\begin{figure}[ht]
    \includegraphics[scale=0.15,trim={0 0cm 5cm 0cm},clip]{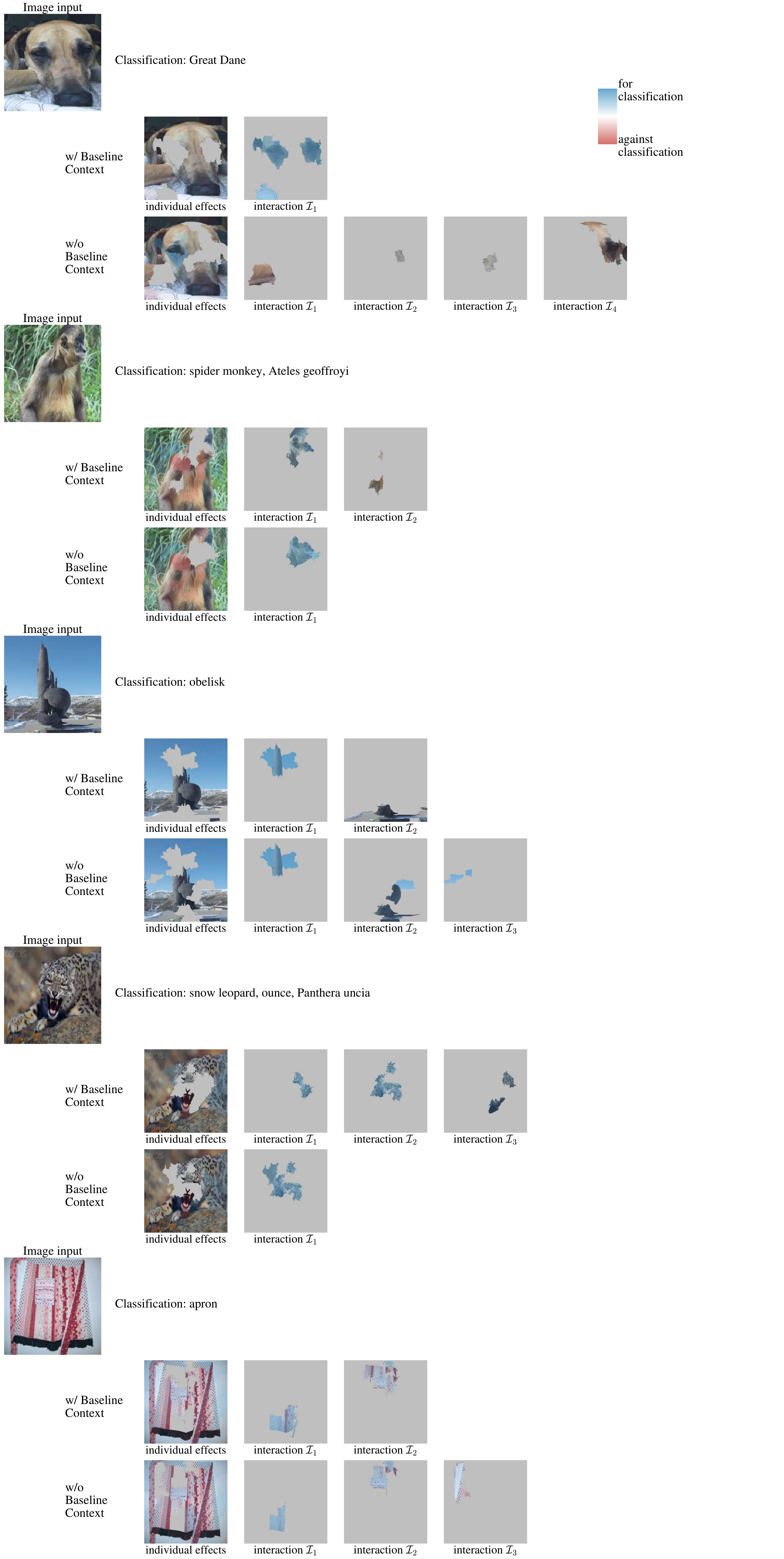}
    
    \caption{Image Viz. with {\archdetect} Ablation A. The interactions tend to focus more on salient patches of the images when including the baseline context, which is proposed in {\archdetect}.
    \label{fig:image_wo_viz_a}
    }
    \vspace{-0.1in}
\end{figure}

\begin{figure}[ht]
    \includegraphics[scale=0.15,trim={0 0cm 5cm 0cm},clip]{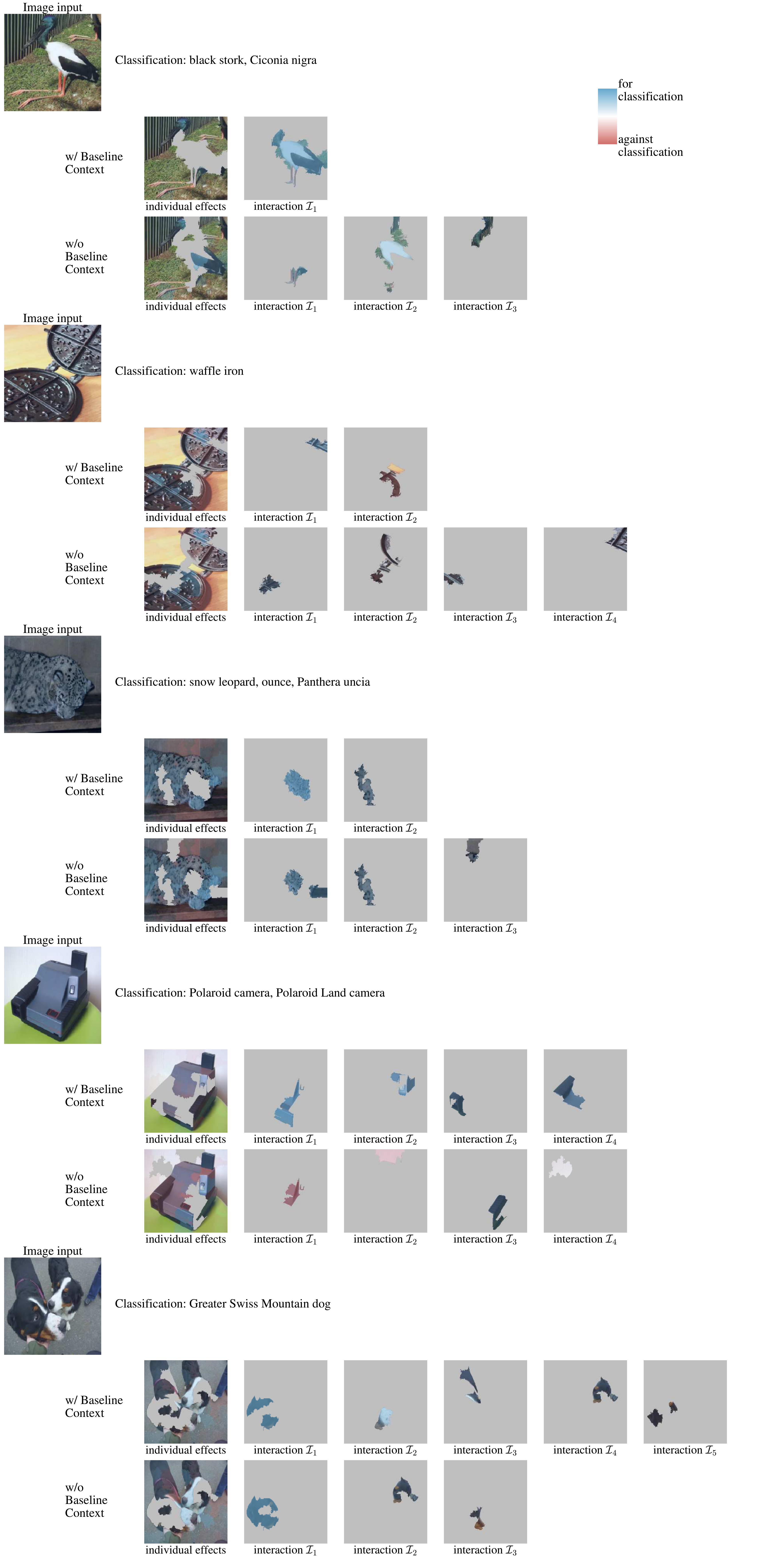}
    
    \caption{Image Viz. with {\archdetect} Ablation B. The interactions tend to focus on salient patches of the images when including the baseline context.
    \label{fig:image_wo_viz_b}
    }
    \vspace{-0.1in}
\end{figure}
\end{appendix}

\end{document}